\documentclass{article}

\usepackage{arxiv}
\usepackage[T1]{fontenc}

\usepackage{amsmath}
\usepackage{amssymb}
\usepackage{amsthm}
\usepackage{mathrsfs}
\usepackage{subcaption}
\usepackage{booktabs}
\usepackage{array}
\usepackage{tabularx}
\newcolumntype{Y}{>{\centering\arraybackslash}X}
\usepackage{multirow}
\usepackage{makecell}
\usepackage{nicefrac}
\usepackage{graphicx}
\usepackage{tikz}
\usepackage{pgfplots}
\usepackage{forest}
\usepackage{xcolor}
\usepackage{placeins}
\usepackage{xurl}
\usepackage[numbers,sort&compress]{natbib}
\usepackage[colorlinks=true,allcolors=blue]{hyperref}
\pgfplotsset{compat=1.17}
\usetikzlibrary{shapes.geometric, arrows.meta, positioning, calc, fit, backgrounds}

\graphicspath{{figures/}}

\newtheorem{theorem}{Theorem}

\newtheorem{proposition}{Proposition}
\newtheorem{definition}{Definition}
\newtheorem{remark}{Remark}

\forestset{
  declare dimen={fork sep}{0.5em},
  forked edge'/.style={
    edge={rotate/.option=!parent.grow},
    edge path'={(!u.parent anchor) -- (.child anchor)},
  },
  forked edge/.style={
    on invalid={fake}{!parent.parent anchor=children},
    child anchor=parent,
    forked edge',
  },
  forked edges/.style={for nodewalk={#1}{forked edge}},
  forked edges/.default=tree,
}

\newcommand{\R}{\mathbb{R}}
\newcommand{\sgn}{\operatorname{sgn}}
\newcommand{\softplus}{\operatorname{softplus}}
\newcommand{\dd}{\mathrm{d}}

\usepackage{authblk}

\title{\Large FHKAN: Fractional Hermite Kolmogorov--Arnold Networks\\for Data-Driven and Physics-Informed Learning}
\renewcommand{\shorttitle}{FHKAN: Fractional Hermite Kolmogorov--Arnold Networks}
\date{}

\author[1]{Arsha Farnoush}
\author[1,2]{Amirhossein Sadr}
\author[2]{Aida Pakniyat}
\author[1]{Dara Rahmati\thanks{Corresponding author. \texttt{d\_rahmati@sbu.ac.ir}}}
\author[3]{Saeid Gorgin}

\author[4,5,6]{Kourosh Parand}
\affil[1]{Department of Computer Science and Engineering, Shahid Beheshti University, Tehran, Iran}
\affil[2]{School of Computer Science, Institute for Research in Fundamental Sciences (IPM), Tehran, Iran}
\affil[3]{School of Physics, Engineering and Computer Science, University of Hertfordshire, Hatfield, UK}
\affil[4]{Department of Computer and Data Sciences, Faculty of Mathematical Sciences, Shahid Beheshti University, Tehran, Iran}
\affil[5]{Department of Cognitive Modeling, Institute for Cognitive and Brain Sciences, Shahid Beheshti University, Tehran, Iran}
\affil[6]{International Business University, Toronto, ON, Canada}

\hypersetup{
  pdftitle={FHKAN: Fractional Hermite Kolmogorov-Arnold Networks for Data-Driven and Physics-Informed Learning},
  pdfauthor={Amirhossein Sadr, Arsha Farnoush, Aida Pakniyat, Dara Rahmati, Kourosh Parand, Saeid Gorgin},
  pdfkeywords={Kolmogorov-Arnold networks, Fractional Hermite functions, Ordinary differential equations, Partial differential equations, Physics-informed machine learning, Activation functions},
}

\begin{document}

\maketitle

\begin{abstract}
Kolmogorov--Arnold Networks (KANs) replace fixed scalar weights by learnable
univariate maps, yet most spectral KANs still rely on Chebyshev
polynomials that live on a bounded interval and therefore require a saturating
squashing of the pre-activation. We introduce the Fractional Hermite Kolmogorov--Arnold Network (FHKAN), a
neural architecture that equips KANs with an
adaptive, trainable fractional-order generalized Hermite basis. With a nonnegative shape parameter, the underlying
generalized Hermite functions are orthonormal on the entire real line and
uniformly bounded under a Gaussian envelope, so they act directly on unbounded
pre-activations without any saturating map. We derive closed-form derivatives
of the generalized, fractional, and scaled basis away from the origin and
determine its order of differentiability at the origin, where the fractional
power map creates a singular point. On
synthetic regression, FHKAN attains the lowest test error at every network
depth among nine fixed activations and the spline-based KAN. On image
classification, image denoising, and sentiment analysis, the best-performing
fractional Hermite block yields the highest accuracy or reconstruction
quality. On physics-informed problems, FHKAN reaches mean absolute errors of order
$10^{-6}$ to $10^{-5}$ on the Lane--Emden equation, residuals of order
$10^{-4}$ on the Burgers equation, and absolute errors below $10^{-4}$ on a
fractional delay differential equation.

\end{abstract}


\section{Introduction}
The expressive power of an artificial neural network is, to a large extent,
determined by the non-linear transformation that is applied to the affine
combination of its inputs. Historically, this transformation has taken the form
of a fixed, scalar activation function shared by every neuron in a layer
\cite{aggarwal2018neural, jagtap2022important}. Saturating choices such as the
sigmoid and the hyperbolic tangent enabled the first generation of deep models
but suffer from vanishing gradients \cite{han1995influence, lecun2002efficient},
whereas the rectified linear unit (ReLU) \cite{nair2010rectified} and its many
descendants---Leaky ReLU \cite{maas2013rectifier}, parametric ReLU
\cite{he2015delving}, GELU \cite{hendrycks2016gaussian}, ELU
\cite{clevert2015fast}, and Mish \cite{misra2019mish}---traded smoothness for
sparsity at the price of dead units. A recurring theme in this line of research
is that the activation function should be adapted to the data, a viewpoint embodied by trainable activations such
as parametric ReLU \cite{he2015delving}, the learnable extended activation
\cite{bodyanskiy2023learnable}, modReLU \cite{arjovsky2016unitary}, and DELU
\cite{pishchik2023trainable}.

Kolmogorov--Arnold Networks (KANs) \cite{liu2024kan} push this idea to its
logical conclusion. Motivated by the Kolmogorov--Arnold representation theorem,
they discard scalar weights altogether and place learnable univariate functions
on the edges of the computational graph. The original formulation parametrizes
these functions through B-splines, and a rapidly growing literature demonstrates
that KANs can match or exceed multi-layer perceptrons in accuracy while using
markedly fewer parameters, with improved interpretability and favorable scaling
behavior \cite{liu2024kan, cheon2024kolmogorov, liu2024ikan, sadr20262, sadr2026pg, genet2024temporal,
abueidda2024deepokan, naraghi2026iga}. The choice of the univariate family is, however, far from
unique: wavelet \cite{bozorgasl2024wav}, Fourier \cite{xu2024fourierkan},
radial-basis \cite{li2024kolmogorov}, temporal \cite{genet2024tkan}, and smooth
\cite{samadi2024smooth} variants have all been proposed, each tailoring the
inductive bias of the network to a particular class of signals.
Table~\ref{tbl:kan-variants} groups these variants by the univariate
family placed on the edges, the natural domain of each family, and the
corresponding references.

\begin{table}[htbp]
\centering
\caption{Summary of Kolmogorov--Arnold Network variants grouped by the
univariate basis family used as edge functions, the natural domain on which
each family is defined, and the corresponding references. Only the proposed
fractional generalized Hermite basis operates natively on the entire real
line.}
\vspace{0.5em}
\small
\begin{tabularx}{\textwidth}{@{}llX@{}}
\toprule
Basis family of the KAN variant & Natural domain & References \\
\midrule
B-splines (original KAN) & bounded, local & \cite{liu2024kan, cheon2024kolmogorov, liu2024ikan, abueidda2024deepokan} \\
Wavelets & bounded/local & \cite{bozorgasl2024wav} \\
Fourier bases & periodic & \cite{xu2024fourierkan} \\
Radial basis functions & local & \cite{li2024kolmogorov} \\
Temporal/recurrent bases & sequential & \cite{genet2024tkan, genet2024temporal} \\
Smooth structural bases & bounded & \cite{samadi2024smooth} \\
Chebyshev polynomials (first kind) & bounded, $[-1,1]$ & \cite{ss2024chebyshev, SynodicMonth2024} \\
Fractional Jacobi functions & bounded, $[-1,1]$ & \cite{aghaei2024fkan} \\
Fractional generalized Hermite functions & unbounded, $\R$ & this work \\
\bottomrule
\end{tabularx}
\label{tbl:kan-variants}
\end{table}

A particularly attractive subfamily replaces splines with orthogonal
polynomials, which combine global support, analytic derivatives, and a sound
approximation theory \cite{shen2011spectral, gautschi2004orthogonal,
singh2012approximation}. Orthogonal polynomials have a long track record as
neural activations: Chebyshev polynomials have been used for classification and
denoising \cite{deepthi2023development} and for control
\cite{alsaadi2023control}, Legendre polynomials for max-pooling surrogates
\cite{venkatappareddy2021legendre}, Laguerre polynomials for forecasting and
sensing \cite{wang2020wind, patra2010development}, Hermite functions for
biomedical signal classification \cite{ebrahimzadeh2016classification,
ma2005constructive, rigatos2006feed}, and Bernstein bases for graph learning
\cite{guo2023graph}. Within the KAN framework, Chebyshev KANs
\cite{ss2024chebyshev, SynodicMonth2024} and fractional KANs based on fractional Jacobi functions
\cite{aghaei2024fkan} have shown that a spectral edge function can be both
accurate and inexpensive. All of these polynomials are special cases of the
generalized hypergeometric function \cite{carlson1991b, benzoubeir2009hypergeometric,
vieira2023bicomplex}, and the choice among them amounts to selecting which
analytic properties to favor.

Two limitations are common to virtually every polynomial KAN proposed so far.
First, classical orthogonal polynomials (Chebyshev, Legendre, Jacobi, and
Gegenbauer polynomials)
are orthogonal only on a \emph{bounded} interval. Because the pre-activation of a
linear layer is unbounded, these methods must first compress it into the
admissible interval through a saturating map such as a sigmoid or a hyperbolic
tangent \cite{ss2024chebyshev, aghaei2024fkan}; this auxiliary non-linearity
re-introduces saturation, distorts the geometry of the input, and effectively
wastes the tails of the distribution. Second, polynomials of integer degree
inherit Runge oscillations, numerical conditioning issues, and a limited ability
to model non-smooth or singular behavior using only a few terms
\cite{shen2011spectral}. Fractional orthogonal functions, obtained by composing a
polynomial with a fractional power map, have proven remarkably effective at
mitigating the latter issue in numerical analysis
\cite{rad2023learning, parand2017accurate, parand2016solving, parand2017fractional,
hajimohammadi2021fractional}, yet they have only begun to be explored inside
KANs.

In this paper we argue that the Hermite family is uniquely suited to address
both limitations at once, and we build a KAN around it. Generalized
Hermite functions are orthonormal on the whole real line, are bounded by a
Gaussian envelope, and decay to zero at infinity; consequently they can be
applied directly to the output of a linear layer with no squashing,
while simultaneously providing the built-in localization that makes splines and
radial-basis functions attractive. By endowing the Hermite basis with a
trainable fractional exponent, a shape parameter, and a scale parameter, we
obtain an adaptive edge function whose smoothness, locality, and effective
spectral order are discovered from the data. We call the resulting building
block the \emph{fractional Hermite neural block} (fHNB) and the network that it
defines the \emph{Fractional Hermite Kolmogorov--Arnold Network} (FHKAN).

The contributions of this work are summarized as follows.
\begin{itemize}
  \item We introduce a trainable fractional-order generalized Hermite basis for
  Kolmogorov--Arnold Networks that operates natively on the unbounded real line,
  removing the squashing step required by bounded-domain polynomial KANs.
  \item We design the fractional Hermite neural block, a self-contained,
  differentiable module with three interpretable parameters---shape $\mu$,
  fractional order $\gamma$, and scale $\lambda$---together with the
  reparametrizations that keep it well posed during gradient-based training, and
  we derive closed-form
  derivatives of the generalized, fractional, and scaled basis for every value
  of the shape parameter, together with the exact order of differentiability
  of each basis order and of the whole block at the origin.
  \item We validate FHKAN on a comprehensive suite of data-driven tasks
  (regression, image classification, image denoising, sentiment analysis) and
  physics-informed tasks (ordinary, partial, and fractional delay differential
  equations). On the
  data-driven tasks FHKAN is compared with fixed activations and with the
  spline-based KAN. On the physics-informed tasks
  the learned solutions are assessed against exact or high-accuracy
  reference solutions.
\end{itemize}

The remainder of the paper is organized as follows. Section~\ref{sec:motivation}
motivates the proposed basis by contrasting bounded and unbounded function
families and by examining the locality--globality trade-off. Section~\ref{sec:background}
reviews multi-layer perceptrons, Kolmogorov--Arnold Networks, and the
generalized Hermite functions on which our method rests,
including the classical spectral approximation theory on $\R$ and on $[0,+\infty)$
that motivates the
collocation experiments. Section~\ref{sec:method}
develops the fractional Hermite neural block and the FHKAN architecture.
Section~\ref{sec:experiments} reports the experimental evaluation, Section~\ref{sec:discussion} provides a discussion of parameter efficiency and
Section~\ref{sec:conclusion} concludes. A summary of the notation used
throughout the paper is collected in Table~\ref{tbl:notation}.

\section{Motivation}
\label{sec:motivation}
Before formalizing the method, we make explicit the design principles that lead
to a Hermite-based KAN. Our reasoning revolves around three properties that an
edge function should ideally satisfy simultaneously: \emph{domain compatibility}
with the unbounded output of a linear layer, a controllable balance between
\emph{global expressiveness} and \emph{local adaptivity}, and \emph{closed-form derivatives} for both backpropagation and physics-informed residuals.

\paragraph{Domain compatibility}
Let $z=\theta^{\top}\boldsymbol\zeta+b$ denote a pre-activation. In general
$z\in\R$, so any basis used as an edge function must be meaningful on the whole
real line. Classical orthogonal polynomials---Chebyshev, Legendre, Jacobi,
and Gegenbauer polynomials---are orthogonal only on a compact interval
and diverge polynomially
outside it. Spectral KANs therefore prepend a bounded map $g:\R\to(d_0,d_1)$,
typically a sigmoid or a hyperbolic tangent, before evaluating the polynomial
\cite{ss2024chebyshev, aghaei2024fkan}. This detour has three undesirable
consequences: it saturates for large $|z|$ and thus discards information carried
by the tails, it composes an additional non-linearity that is not adapted to the
task, and it introduces a fixed length scale unrelated to the data. A basis that
is square-integrable on the entire real line removes the need for $g$ altogether.
The Hermite functions $\{\psi_n\}$ possess exactly this property: they are an
orthonormal basis of $L^2(\R)$, they are uniformly bounded, and they decay to
zero at infinity, so they can be evaluated directly at any $z\in\R$.
The same functions are used on $(-\infty,+\infty)$ and, through the map
in Section~\ref{sec:hermite-spectral}, on $[0,+\infty)$.

\paragraph{Locality versus globality}
Spline and radial-basis KANs owe their flexibility to \emph{local} support: each
basis element influences only a neighborhood, which makes them robust but also
non-smooth and parameter-hungry. Orthogonal polynomials sit at the opposite
extreme: they are perfectly smooth and \emph{global}, but a localized feature
must be reconstructed from many high-degree terms, which invites oscillations.
Hermite functions occupy a fertile middle ground. The Gaussian envelope
$e^{-x^2/2}$ confines the effective support of $\psi_n$ to an interval whose
width grows only as $\sqrt{n}$, so low-order Hermite functions behave like smooth
bumps---comparable to radial-basis functions---while higher orders progressively
recover global, oscillatory detail. Endowing the basis with a trainable scale
$\lambda$ lets the network dilate or contract this effective support, and a
trainable fractional order $\gamma$ continuously deforms the otherwise integer
spectral content. Figure~\ref{fig:frac-hermite} illustrates how the fractional
order reshapes the basis without leaving the real line or violating boundedness.

\paragraph{Closed-form derivatives}
Gradient-based training and physics-informed residuals both require derivatives
of the edge function. Classical Hermite functions satisfy the ladder (creation and
annihilation) relations of the quantum harmonic oscillator, so their derivatives
are expressed in closed form as a combination of neighboring Hermite functions
of the same family. For
the generalized, fractional, and scaled basis used by FHKAN, a relation of the
same type holds away from the origin (Section~\ref{sec:block-regularity}).
At the origin the fractional power map has an unbounded derivative, and the
order of differentiability of each edge function is determined there
by its shape and fractional parameters. This regularity matters most for
physics-informed residuals that contain second derivatives.

Figure~\ref{fig:taxonomy} situates the Hermite family within the hypergeometric
hierarchy and highlights the branch---defined on $\R$ with a Gaussian weight and
controlled by the reflection parameter $\mu$---that the present work exploits.
Whereas previous KANs traveled down the Chebyshev-polynomial branch
(bounded support, weight
$(1-x)^\alpha(1+x)^\beta$), FHKAN travels down the Hermite-function
branch (real line,
weight $|x|^{2\mu}e^{-x^2}$), and recovers the standard Hermite functions when
the shape parameter $\mu$ vanishes.

{\small
\begin{figure}[htbp]
  \centering
  \resizebox{0.96\textwidth}{!}{
  \begin{forest}
    for tree={grow'=-90, draw, rounded corners, align=center, inner sep=5pt,
      l=18pt, l sep=16pt, font=\small},
    forked edges,
    [\textbf{Generalized Hypergeometric Function}\\${}_p\mathcal{F}_q$
      [Gaussian\\Hypergeometric\\${}_2\mathcal{F}_1$
        [{Jacobi polynomials$(\alpha,\beta)$}\\{\footnotesize bounded, $[-1,1]$}
          [{Gegenbauer}\\{$\alpha=\beta$}
            [{Legendre (first kind)}\\{$\alpha=\beta=0$}]
            [{Chebyshev (first kind)}\\{$\alpha=\beta=-\nicefrac12$}]
          ]
        ]
      ]
      [Confluent\\Hypergeometric\\${}_1\mathcal{F}_1$
        [{Laguerre$(\nu)$}\\{half line, $[0,\infty)$}]
        [{\textbf{Generalized Hermite}}$(\mu)$\\{\textbf{real line}, $\R$, weight $|x|^{2\mu}e^{-x^2}$}, draw, very thick
          [{Hermite}\\{$\mu=0$}, draw, very thick]
        ]
      ]
    ]
  \end{forest}}
  \caption{Position of the generalized Hermite family within the hypergeometric
  hierarchy. Earlier spectral KANs operate on the Chebyshev-polynomial
  branch, whose members
  are orthogonal only on a bounded interval. FHKAN operates on the Hermite
  branch (highlighted), which is orthogonal on the entire real line under the
  weight $|x|^{2\mu}e^{-x^2}$ and reduces to the classical Hermite functions for
  $\mu=0$. The Chebyshev and Legendre polynomials displayed in the
  Jacobi-polynomial branch are of the first kind.}
  \label{fig:taxonomy}
\end{figure}
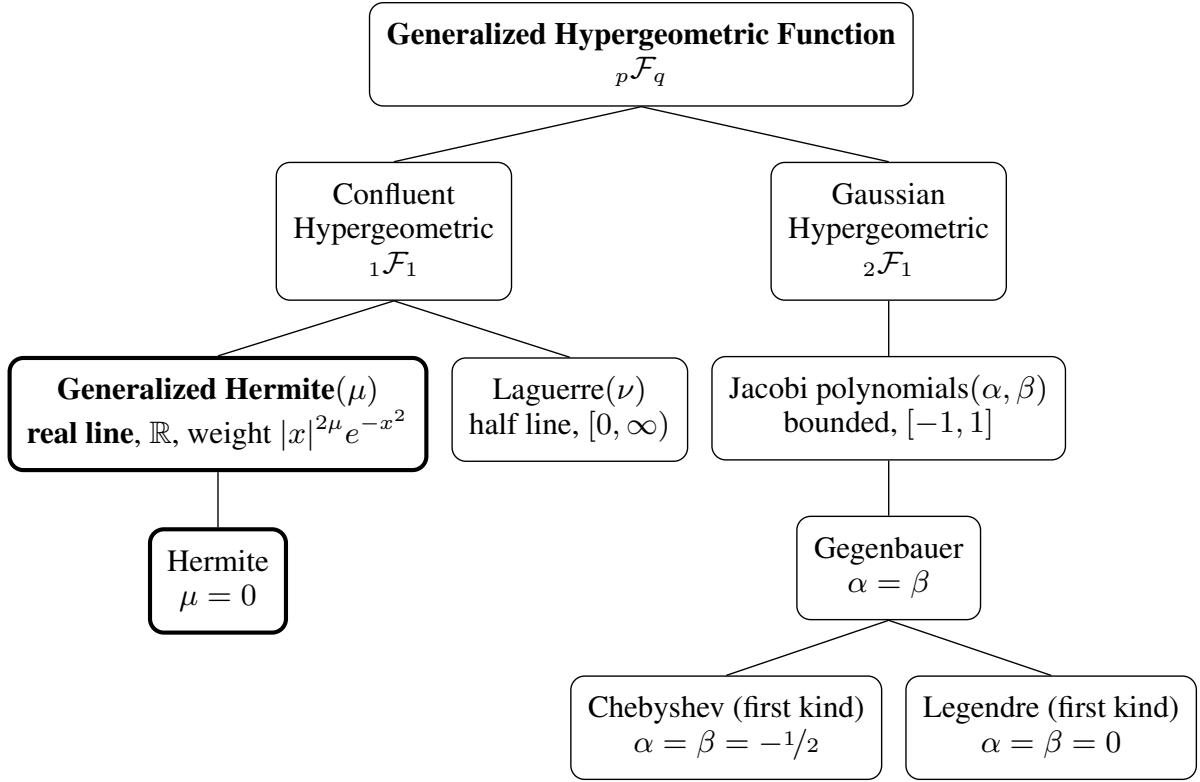}

\section{Background}
\label{sec:background}
This section fixes notation and recalls the three ingredients underlying FHKAN:
the multi-layer perceptron, the Kolmogorov--Arnold network, and the family of
generalized Hermite functions.

Lower-case symbols denote univariate objects: $\psi_n^{(\mu)}$ and
$\Psi_n^{(\mu,\gamma)}$ are the generalized Hermite function and its fractional
counterpart, and $\phi_q$ is one edge function. Upper-case symbols denote
network-level maps: $\Phi$ combines the edges of a Kolmogorov--Arnold layer
and $\hat\chi$ is the map realized by the whole network. In the acronyms,
fHNB is one block and FHKAN is the network built from those blocks.
Table~\ref{tbl:notation} lists the symbols. In this section the symbol
$\psi$ is also used for the output map of a multi-layer perceptron; from
Section~\ref{sec:hermite} onward, $\psi_n^{(\mu)}$ is the Hermite function.

\begin{table}[htbp]
\centering
\caption{Summary of the notation. Lower-case symbols denote univariate
functions and upper-case symbols denote network-level maps.}
\vspace{0.5em}
\footnotesize
\begin{tabularx}{\textwidth}{@{}llX@{}}
\toprule
Symbol & Type & Meaning \\
\midrule
$\boldsymbol\zeta,\ \zeta_p$ & vector, scalar & input of the network and its $p$-th component \\
$\chi,\ \hat\chi$ & network map & target function and the map realized by the network \\
$\psi$ & network map & task-dependent output map of a multi-layer perceptron
  (distinct from the Hermite functions $\psi_n^{(\mu)}$) \\
$\Theta^{(i)},\ \mathbf{b}^{(i)}$ & parameters & weight matrix and bias of layer $i$ \\
$\theta_{p,q}$ & parameter & expansion coefficient of order $q$ on input $p$ \\
$\sigma$ & univariate & fixed scalar activation, or the logistic sigmoid \\
$\phi_{q,p},\ \phi_q$ & univariate & inner (edge) functions of a Kolmogorov--Arnold layer \\
$\Phi_q,\ \boldsymbol\Phi_i$ & network map & outer functions of a Kolmogorov--Arnold layer \\
$H_n^{(\mu)}$ & univariate & generalized Hermite polynomial of degree $n$ \\
$\psi_n^{(\mu)}$ & univariate & generalized Hermite function of order $n$ \\
$\Psi_n^{(\mu,\gamma)}$ & univariate & fractional generalized Hermite function \\
$\vartheta_\gamma$ & univariate & odd fractional map $\sgn(x)|x|^{\gamma}$ \\
$w_\mu,\ \varpi_\gamma$ & univariate & orthogonality weights of $H_n^{(\mu)}$ and $\Psi_n^{(\mu,\gamma)}$ \\
$\varrho_n^{(\mu)}$ & constant & normalization constant of $H_n^{(\mu)}$ \\
$\mu,\ \gamma,\ \lambda$ & parameters & unconstrained shape, fractional order and scale of an fHNB \\
$\hat\mu,\ \hat\gamma,\ \hat\lambda$ & parameters & their constrained counterparts, see \eqref{eq:fhnb} \\
$Q,\ q$ & hyperparameter, index & largest order of an fHNB block and the running order \\
$\omega$ & scalar & mapped collocation coordinate \\
$\mathcal{H}_Q$ & space & span of $\{\psi_0^{(\mu)},\dots,\psi_Q^{(\mu)}\}$ \\
$P_Q$ & operator & $L^2(\R)$ orthogonal projection onto $\mathcal{H}_Q$ \\
$I_Q$ & operator & interpolation at the Hermite--Gauss nodes of $H_{Q+1}$ \\
$\widehat\psi_n$ & univariate & Pulled back Hermite function \\
$\mathcal{N},\ \mathcal{S},\ \mathfrak{R}$ & operators & differential operator, source term and residual \\
$\alpha$ & constant & order of the Caputo fractional derivative of the equation \\
$m$ & constant & polytropic index of the Lane--Emden equation \\
\bottomrule
\end{tabularx}
\label{tbl:notation}
\end{table}

\subsection{Multi-layer perceptrons}
A multi-layer perceptron (MLP) approximates an unknown map
$\chi:\R^{d}\to\R$ by alternating affine transformations with a fixed scalar
non-linearity $\sigma(\cdot)$. Writing $\mathbf{h}^{(0)}=\boldsymbol\zeta$, the
network is the composition
\begin{equation*}
  \begin{aligned}
  \mathbf{h}^{(i)} &= \sigma\!\left(\Theta^{(i)}\mathbf{h}^{(i-1)} + \mathbf{b}^{(i)}\right), \quad i = 1,\dots,L-1,\\
  \hat{\chi}(\boldsymbol\zeta) &= \psi\!\left(\Theta^{(L)}\mathbf{h}^{(L-1)} + \mathbf{b}^{(L)}\right),
  \end{aligned}
\end{equation*}
where $\Theta^{(i)}$ and $\mathbf{b}^{(i)}$ are the learnable weights and biases
and $\psi$ is a task-dependent output map. The universal approximation theorem
guarantees that, under mild conditions on $\sigma$, a sufficiently wide single
hidden layer can approximate any continuous function on a compact set to
arbitrary accuracy \cite{pinkus1999approximation, gripenberg2003approximation,
lu2017expressive, kidger2020universal, nishijima2021universal, kumagai2020universal};
depth then improves the rate of this approximation in practice. Because the
gradients propagated during backpropagation are modulated by $\sigma'(\cdot)$,
a useful activation function should be differentiable, non-linear, computationally
cheap, and free of vanishing or exploding gradients \cite{jagtap2022important}.

\subsection{Kolmogorov--Arnold networks}
Kolmogorov--Arnold networks replace the fixed scalar activation by learnable
univariate functions placed on the edges of the graph. They are motivated by the
Kolmogorov--Arnold representation theorem
\cite{kolmogorov1957representation, braun2009constructive, sprecher1965structure,
lorentz1962metric}, which states that any continuous
$\chi:[0,1]^{d}\to\R$ admits the exact decomposition
\begin{equation}
  \chi(\boldsymbol\zeta) = \sum_{q=0}^{2d}\Phi_q\!\left(\sum_{p=1}^{d}\phi_{q,p}(\zeta_p)\right),
  \label{eq:kart}
\end{equation}
for suitable continuous outer functions $\Phi_q:\R\to\R$ and inner functions
$\phi_{q,p}:[0,1]\to\R$.
Interpreting the right-hand side of \eqref{eq:kart} as a
two-layer network with learnable univariate maps, and stacking such layers,
yields a deep KAN
\begin{equation*}
  \hat{\chi}(\boldsymbol\zeta) = \big(\boldsymbol\Phi_{n-1}\circ\cdots\circ\boldsymbol\Phi_{1}\circ\boldsymbol\Phi_{0}\big)(\boldsymbol\zeta),
\end{equation*}
whose convergence and smoothness have been analyzed in
\cite{liu2024kan, schmidt2021kolmogorov, dzhenzher2021structured}. While the
original KAN parametrizes the univariate maps with B-splines
\cite{forster2011splines, manni2017generalized}, a computationally lighter route
expands them in a fixed orthogonal system and learns only the expansion
coefficients, an architecture sometimes called a learnable activation network.
The Chebyshev KAN \cite{ss2024chebyshev, SynodicMonth2024} is the prototypical
example,
\begin{equation*}
  \hat{\chi}(\boldsymbol\zeta) = \sum_{q=1}^{n}\sum_{p=1}^{d}\theta_{p,q}\,T_q\!\big(\tanh(\zeta_p)\big),
\end{equation*}
where $T_q$ is the Chebyshev polynomial of the first kind of degree $q$
and the hyperbolic tangent
maps the input into the admissible interval $[-1,1]$.

\subsection{Generalized Hermite functions}
\label{sec:hermite}
We now collect the properties of the Hermite family that make it a compelling
basis for KANs. The generalized Hermite polynomials $\{H_n^{(\mu)}\}_{n\ge0}$,
introduced in the context of the Bose-like oscillator and the Dunkl calculus
\cite{rosenblum1994generalized, dunkl1989differential, chihara2011introduction},
are defined for a reflection parameter $\mu>-\tfrac12$ by the three-term
recurrence
\begin{equation}
  \begin{aligned}
    H_0^{(\mu)}(x) &= 1, \qquad H_1^{(\mu)}(x) = 2x,\\
    H_{n+1}^{(\mu)}(x) &= 2x\,H_n^{(\mu)}(x) - 2\,\big(n + \mu\,[1-(-1)^n]\big)\,H_{n-1}^{(\mu)}(x), \quad n\ge1.
  \end{aligned}
  \label{eq:hermite-recurrence}
\end{equation}
For $\mu=0$ the bracketed term collapses to $n$ and \eqref{eq:hermite-recurrence}
reduces to the classical (physicists') Hermite recurrence. These polynomials are
even or odd according to the parity of $n$, and they reduce to generalized
Laguerre polynomials under the quadratic substitution $x\mapsto x^2$,
\begin{equation*}
  \begin{aligned}
  H_{2k}^{(\mu)}(x) &= (-1)^k\,2^{2k}\,k!\,L_k^{(\mu-\nicefrac12)}(x^2),\\
  H_{2k+1}^{(\mu)}(x) &= (-1)^k\,2^{2k+1}\,k!\,x\,L_k^{(\mu+\nicefrac12)}(x^2),
  \end{aligned}
\end{equation*}
which explains their placement on the confluent branch of
Figure~\ref{fig:taxonomy}.

\begin{theorem}[Orthogonality]\label{thm:orth}
For $\mu>-\tfrac12$ the generalized Hermite polynomials are orthogonal on $\R$
with respect to the weight $w_\mu(x)=|x|^{2\mu}e^{-x^2}$:
\begin{equation*}
  \int_{-\infty}^{\infty} H_m^{(\mu)}(x)\,H_n^{(\mu)}(x)\,|x|^{2\mu}e^{-x^2}\,\dd x
  = \varrho_n^{(\mu)}\,\delta_{m,n},
\end{equation*}
with normalization constants
$\varrho_{2k}^{(\mu)} = 2^{4k}\,k!\,\Gamma\!\big(k+\mu+\tfrac12\big)$ for even
indices and
$\varrho_{2k+1}^{(\mu)} = 2^{4k+2}\,k!\,\Gamma\!\big(k+\mu+\tfrac32\big)$ for odd
indices. The normalized polynomials
$H_n^{(\mu)}/\sqrt{\varrho_n^{(\mu)}}$ are orthonormal in the weighted space
$L^2_{w_\mu}(\R)$. Folding the square root of the weight into these
polynomials produces the generalized Hermite functions
\eqref{eq:hermite-functions}, which are orthonormal in $L^2(\R)$.
\end{theorem}

To obtain a basis that is bounded on the entire real line we fold the square
root of the weight into the polynomial and define the \emph{generalized Hermite
functions}
\begin{equation}
  \psi_n^{(\mu)}(x) = \frac{1}{\sqrt{\varrho_n^{(\mu)}}}\,|x|^{\mu}\,e^{-x^2/2}\,H_n^{(\mu)}(x),
  \qquad x\in\R .
  \label{eq:hermite-functions}
\end{equation}

\begin{theorem}[Orthonormality and boundedness]\label{thm:onb}
For every $\mu>-\tfrac12$ the functions $\{\psi_n^{(\mu)}\}_{n\ge0}$ form an
orthonormal system in $L^2(\R)$, that is
$\int_{-\infty}^{\infty}\psi_m^{(\mu)}(x)\psi_n^{(\mu)}(x)\,\dd x=\delta_{m,n}$.
If $\mu\ge0$, each $\psi_n^{(\mu)}$ is uniformly bounded on $\R$ and satisfies
$\psi_n^{(\mu)}(x)\to0$ as $|x|\to\infty$.
If $-\tfrac12<\mu<0$, the factor $|x|^{\mu}$ is singular at the origin, so
$\psi_n^{(\mu)}$ is not uniformly bounded.
\end{theorem}

Orthonormality holds for every $\mu>-\tfrac12$, including the classical value
$\mu=0$. Uniform boundedness on $\R$ holds for $\mu\ge0$. In the trainable
block of Section~\ref{sec:method}, the
shape parameter is reparametrized as $\hat\mu=\softplus(\mu)>0$, so every
trainable value lies in the range $\hat\mu\ge0$ where the uniform bound holds;
the classical value $\hat\mu=0$ is the limit $\mu\to-\infty$. The range
$-\tfrac12<\hat\mu<0$, in which the factor $|x|^{\hat\mu}$ is singular at the
origin, is excluded by construction.

The uniform bound and the decay at infinity let Hermite functions act as edge
functions on the whole line. For the classical case
$\mu=0$ the derivatives are available through the harmonic-oscillator ladder
operators
$a=\tfrac{1}{\sqrt2}\big(x+\tfrac{\dd}{\dd x}\big)$ and
$a^{\dagger}=\tfrac{1}{\sqrt2}\big(x-\tfrac{\dd}{\dd x}\big)$, which act as
$a\,\psi_n=\sqrt{n}\,\psi_{n-1}$ and $a^{\dagger}\psi_n=\sqrt{n+1}\,\psi_{n+1}$,
whence the following closed-form derivative.

\begin{proposition}[Ladder derivative, classical case $\mu=0$]\label{prop:ladder}
The classical Hermite functions satisfy
\begin{equation}
  \frac{\dd}{\dd x}\,\psi_n(x)
  = \sqrt{\tfrac{n}{2}}\;\psi_{n-1}(x) - \sqrt{\tfrac{n+1}{2}}\;\psi_{n+1}(x).
  \label{eq:ladder-derivative}
\end{equation}
Consequently every derivative of a Hermite function is again a finite
combination of Hermite functions of the same family.
\end{proposition}

Property~\eqref{eq:ladder-derivative} mirrors the self-referential derivative of
Jacobi polynomials.
It holds on the whole real line and maps bounded functions to bounded
functions. The trainable block of Section~\ref{sec:method} uses a shape
parameter $\mu\neq0$, for which the derivative acquires an additional term.

\begin{proposition}[Derivative of generalized Hermite functions]\label{prop:gen-derivative}
Let $\mu>-\tfrac12$, set $\theta_n=n+\mu\,[1-(-1)^n]$ for $n\ge0$, and let
$\psi_{-1}^{(\mu)}\equiv0$. For every $n\ge0$ and every $x\neq0$,
\begin{equation}
  \frac{\dd}{\dd x}\,\psi_n^{(\mu)}(x)
  = \sqrt{\tfrac{\theta_n}{2}}\;\psi_{n-1}^{(\mu)}(x)
  - \sqrt{\tfrac{\theta_{n+1}}{2}}\;\psi_{n+1}^{(\mu)}(x)
  + (-1)^n\,\frac{\mu}{x}\,\psi_n^{(\mu)}(x).
  \label{eq:gen-derivative}
\end{equation}
For $\mu=0$ the last term vanishes, $\theta_n=n$, and \eqref{eq:gen-derivative}
reduces to \eqref{eq:ladder-derivative}.
\end{proposition}
\begin{proof}
For $x\neq0$ the factor $|x|^{\mu}$ is smooth with
$(|x|^{\mu})'=\mu|x|^{\mu}/x$, so differentiating
\eqref{eq:hermite-functions} gives
\begin{equation*}
  \frac{\dd}{\dd x}\,\psi_n^{(\mu)}(x)
  = \frac{|x|^{\mu}e^{-x^2/2}}{\sqrt{\varrho_n^{(\mu)}}}
    \Big[\big(H_n^{(\mu)}\big)'(x) - x\,H_n^{(\mu)}(x) + \frac{\mu}{x}\,H_n^{(\mu)}(x)\Big].
\end{equation*}
The Laguerre representation of $H_n^{(\mu)}$, combined with the identities
$\frac{\dd}{\dd t}L_k^{(\alpha)}(t)=-L_{k-1}^{(\alpha+1)}(t)$,
$t\frac{\dd}{\dd t}L_k^{(\alpha)}(t)=kL_k^{(\alpha)}(t)-(k+\alpha)L_{k-1}^{(\alpha)}(t)$, and
$L_k^{(\alpha)}-L_{k-1}^{(\alpha)}=L_k^{(\alpha-1)}$ \cite{chihara2011introduction},
yields
\begin{equation*}
  \big(H_{2k}^{(\mu)}\big)' = 4k\,H_{2k-1}^{(\mu)},\qquad
  \big(H_{2k+1}^{(\mu)}\big)' = 2(2k+1)\,H_{2k}^{(\mu)} + \frac{8k\mu}{x}\,H_{2k-1}^{(\mu)}.
\end{equation*}
The recurrence \eqref{eq:hermite-recurrence} reads
$xH_n^{(\mu)}=\tfrac12H_{n+1}^{(\mu)}+\theta_nH_{n-1}^{(\mu)}$. Substituting
these relations, and using
$H_{2k+1}^{(\mu)}=2xH_{2k}^{(\mu)}-4kH_{2k-1}^{(\mu)}$ for odd indices, the
bracket equals
$-\tfrac12H_{n+1}^{(\mu)}+\theta_nH_{n-1}^{(\mu)}+(-1)^n\mu H_n^{(\mu)}/x$ for
both parities. Theorem~\ref{thm:orth} gives
$\varrho_n^{(\mu)}/\varrho_{n-1}^{(\mu)}=2\theta_n$, which converts the
polynomials into the normalized functions and yields \eqref{eq:gen-derivative}.
\end{proof}

Near the origin $\psi_n^{(\mu)}(x)$ behaves like $|x|^{\mu}$ for even $n$ and
like $|x|^{\mu}x$ for odd $n$. Hence the last term of \eqref{eq:gen-derivative}
is bounded for odd $n$, whereas for even $n$ and $0<\mu<1$ it grows like
$|x|^{\mu-1}$ as $x\to0$. Section~\ref{sec:block-regularity} combines
\eqref{eq:gen-derivative} with the fractional map and determines the
regularity of the trainable block at the origin.

\subsection{Spectral approximation on unbounded domains}
\label{sec:hermite-spectral}
The construction above is the neural counterpart of a well-developed spectral
theory for Hermite functions on $\R$
\cite{parand2010hermite,shen2009some,shen2011spectral,tang1993hermite}. We
record the fragments of that theory that motivate the design of FHKAN, rewritten in the
notation of Table~\ref{tbl:notation}, so that the trainable block of
Section~\ref{sec:method} reads as a fractional, data-adaptive generalization of a
classical collocation scheme.

\paragraph{Polynomials versus functions}
The physicists' Hermite polynomials themselves are a poor computational basis:
their envelope grows like $n^{n/2}e^{x^2/2}$ and they oscillate with a
wavelength of order $n^{-1/2}$
\cite{parand2010hermite,shen2009some}. Folding the square root of the
Gaussian weight, as in \eqref{eq:hermite-functions}, cancels that growth. The
resulting functions decay,
$|\psi_n(x)|\to0$ as $|x|\to\infty$. For large $n$ the oscillation in the bulk
has amplitude of order $n^{-1/4}$, while the uniform norm is of order
$n^{-1/12}$. This is why FHKAN uses Hermite
functions as edge maps: they remain in $L^2(\R)$ and stay uniformly
bounded, whereas a polynomial KAN would have to squash its pre-activation
before the basis could even be evaluated.

\paragraph{Function recurrence and derivative Gram matrix}
The three-term relation \eqref{eq:hermite-recurrence} passes to the classical
($\mu=0$) functions, which may be generated stably from
\begin{equation}
  \begin{aligned}
  \psi_{n+1}(x)
    &= x\sqrt{\frac{2}{n+1}}\,\psi_n(x)
     - \sqrt{\frac{n}{n+1}}\,\psi_{n-1}(x),\\
  \psi_0(x)
    &= \pi^{-1/4}e^{-x^2/2},
    \qquad
  \psi_1(x)
    = \sqrt{2}\,x\,\psi_0(x).
  \end{aligned}
  \label{eq:hermite-fn-recurrence}
\end{equation}
Combined with Proposition~\ref{prop:ladder}, the only nonzero entries of the
derivative Gram matrix lie on the main diagonal and on the second off-diagonals.
Writing $\langle u,v\rangle=\int_{\R}uv\,\dd x$,
\begin{equation}
  \bigl\langle \psi_n',\psi_m'\bigr\rangle
  =
  \begin{cases}
    n+\tfrac12, & m=n,\\[2pt]
    -\tfrac12\sqrt{n(n-1)}, & m=n-2,\\[2pt]
    -\tfrac12\sqrt{(n+1)(n+2)}, & m=n+2,\\[2pt]
    0, & \text{otherwise}.
  \end{cases}
  \label{eq:hermite-deriv-gram}
\end{equation}
Thus every $H^1$ inner product that appears in a physics-informed residual
built from classical Hermite functions reduces to a handful of neighboring coefficients,
with no need to differentiate a squashing map.

\paragraph{Hermite--Gauss quadrature}
Let $\{x_j\}_{j=0}^{N}$ denote the roots of $H_{N+1}$, equivalently the
eigenvalues of the symmetric tridiagonal Jacobi matrix associated with
\eqref{eq:hermite-recurrence} (the Golub--Welsch formula). The companion
weights
\begin{equation}
  \widetilde{w}_j
  = \frac{\sqrt{\pi}}{(N+1)\,\widetilde{H}_N(x_j)^2},
  \qquad
  \widetilde{H}_n(x)=\frac{e^{-x^2/2}}{\sqrt{2^n n!}}\,H_n(x),
  \label{eq:hermite-gauss-weights}
\end{equation}
make the quadrature $\sum_{j=0}^{N}p(x_j)\widetilde{w}_j$ exact for every
integrand in the $(2N+2)$-dimensional space of Gaussian-weighted polynomials of
degree at most $2N+1$ \cite{parand2010hermite,shen2011spectral}. In particular
it is exact for products of Hermite functions of orders adding up to at most
$2N+1$. The same nodes, pushed through the map below, are the collocation
set used in Section~\ref{sec:ode}.

\paragraph{Orthogonal projection and approximation rate}
Write $\mathcal{H}_Q=\operatorname{span}\{\psi_0,\dots,\psi_Q\}$ and let
$P_Q:L^2(\R)\to\mathcal{H}_Q$ be the orthogonal projection. The natural
smoothness scale for this family is the harmonic-oscillator Sobolev space
$H_A^r(\R)$, whose norm controls
$(x^2+1)^{(r-k)/2}\partial_x^k\nu$ for $k=0,\dots,r$
\cite{shen2009some,guo2003hermite}. For $0\le s\le r$,
\begin{equation}
  \bigl\|\nu-P_Q\nu\bigr\|_{H^s(\R)}
  \;\le\;
  C\,Q^{(s-r)/2}\,\|\nu\|_{H_A^r(\R)}.
  \label{eq:hermite-approx-rate}
\end{equation}
The exponent $(s-r)/2$ is the rate of the orthogonal projection onto
$\mathcal{H}_Q$ \cite[Theorem~2.3]{guo2003hermite}. It improves as the
oscillator regularity $r$ increases, and the bound tends to zero as
$Q\to\infty$ whenever $r>s$. The estimate applies to
expansions in the fixed classical system $\{\psi_0,\dots,\psi_Q\}$.
Interpolation $I_Q$ at the
Hermite--Gauss nodes of $H_{Q+1}$ obeys the same bound with an aliasing factor
$Q^{1/3}$,
\begin{equation}
  \bigl\|\nu-I_Q\nu\bigr\|_{H^s(\R)}
  \;\le\;
  C\,Q^{1/3+(s-r)/2}\,\|\nu\|_{H_A^r(\R)},
  \label{eq:hermite-interp-rate}
\end{equation}
for $r\ge1$ and $0\le s\le r$ \cite[Theorem~3.1]{guo2003hermite}. The factor
$Q^{1/3}$ is the aliasing contribution of $I_Q$. Both rates are algebraic in
the Sobolev regularity of $\nu$.
The trainable scale $\hat\lambda$ of
Section~\ref{sec:method} plays the role of the classical domain dilation
$\nu(\zeta)\mapsto\nu(\zeta/\ell)$ used by Tang for Gaussian-type targets
\cite{tang1993hermite}. FHKAN learns this dilation jointly with the expansion
coefficients.

\paragraph{Pull-back to the half line}
Hermite functions are orthonormal on $\R$, while many of the differential
equations treated in Section~\ref{sec:experiments} are posed on
$[0,+\infty)$. A standard remedy, used by Parand et al.
for Hermite collocation of Lane--Emden type equations
\cite{parand2010hermite}, is to pull the basis back by the bijection
\begin{equation}
  \omega=\ln\bigl(\sinh(k\zeta)\bigr),\qquad
  \zeta=\frac1k\ln\bigl(e^{\omega}+\sqrt{e^{2\omega}+1}\bigr),
  \qquad k>0,
  \label{eq:dehghan-map}
\end{equation}
which sends $(0,+\infty)$ onto $\R$. The pulled-back family
$\widehat\psi_n(\zeta)=\psi_n\bigl(\omega(\zeta)\bigr)$ is mutually orthogonal
and complete in the weighted space $L^2_w(0,+\infty)$ with
$w(\zeta)=k\coth(k\zeta)$:
\begin{equation}
  \int_0^{\infty}
  \widehat\psi_n(\zeta)\,\widehat\psi_m(\zeta)\,w(\zeta)\,\dd\zeta
  =\delta_{n,m}.
  \label{eq:mapped-orth}
\end{equation}
The change of variables $\dd\omega=k\coth(k\zeta)\,\dd\zeta$ produces the
factor $k$. The weight is singular at the origin and tends to $k$ at infinity.
The same Jacobian clusters the mapped nodes where a
Lane--Emden coefficient $2/\zeta$ is most severe. Differentiating the chart
gives $\dd\omega/\dd\zeta=k\coth(k\zeta)$, which, composed with
\eqref{eq:ladder-derivative}, supplies closed-form $\zeta$-derivatives of
every classical mapped basis function. FHKAN evaluates the fHNB at $\omega$. Under
\eqref{eq:dehghan-map}, negative $\omega$ corresponds to $\zeta$ near the
origin.

With these properties established, we are ready to construct the fractional
Hermite neural block.

\section{Fractional Hermite Kolmogorov--Arnold Networks}
\label{sec:method}
We now introduce FHKAN. We first define a fractional extension of the Hermite
functions, then turn the resulting basis into a trainable neural block, and
finally describe the two network topologies used in our experiments.

\subsection{Fractional Hermite functions}
The integer index $n$ in \eqref{eq:hermite-functions} controls the spectral
content of the basis in discrete steps. To obtain a continuous, data-driven
control over this content---and to enrich the basis with the non-polynomial,
fractional behavior that is known to accelerate the approximation of singular
and non-smooth targets \cite{rad2023learning, parand2017fractional,
parand2017accurate, hajimohammadi2021fractional}---we compose the Hermite
functions with a fractional power map that preserves the real line and the parity
structure.

\begin{definition}[Fractional map]\label{def:frac-map}
For a fractional order $\gamma>0$, the map
\begin{equation*}
  \vartheta_\gamma(x) = \sgn(x)\,|x|^{\gamma}, \qquad x\in\R,
\end{equation*}
is an odd, monotonically increasing bijection of $\R$ onto itself.
\end{definition}

Because $\vartheta_\gamma$ is odd, it commutes with the even/odd parity of the
Hermite functions and keeps every value real for $x\in\R$.

\begin{definition}[Fractional Hermite functions]\label{def:frac-hermite}
For $\mu>-\tfrac12$ and $\gamma>0$, the fractional generalized Hermite functions
are
\begin{equation}
  \Psi_n^{(\mu,\gamma)}(x) = \psi_n^{(\mu)}\!\big(\vartheta_\gamma(x)\big)
  = \psi_n^{(\mu)}\!\big(\sgn(x)\,|x|^{\gamma}\big).
  \label{eq:frac-hermite}
\end{equation}
\end{definition}

\begin{theorem}[Weighted orthonormality]\label{thm:frac-orth}
The fractional Hermite functions form an orthonormal system on $\R$ with respect
to the weight $\varpi_\gamma(x)=\gamma\,|x|^{\gamma-1}$:
\begin{equation*}
  \int_{-\infty}^{\infty}\Psi_m^{(\mu,\gamma)}(x)\,\Psi_n^{(\mu,\gamma)}(x)\,\gamma\,|x|^{\gamma-1}\,\dd x
  = \delta_{m,n}.
\end{equation*}
\end{theorem}
\begin{proof}
Substituting $u=\vartheta_\gamma(x)=\sgn(x)|x|^{\gamma}$ gives
$\dd u=\gamma|x|^{\gamma-1}\dd x$ and maps $\R$ bijectively onto $\R$. Hence the
left-hand side equals
$\int_{-\infty}^{\infty}\psi_m^{(\mu)}(u)\,\psi_n^{(\mu)}(u)\,\dd u=\delta_{m,n}$
by Theorem~\ref{thm:onb}.
\end{proof}

Figure~\ref{fig:frac-hermite} plots the first fractional Hermite functions and
shows how the order $\gamma$ continuously bends the basis: at $\gamma=1$ the
classical Hermite functions are recovered, while $\gamma<1$ steepens the
profile near the origin and stretches the tails, producing genuinely
non-polynomial shapes that remain bounded and square-integrable. This is the
fractional analogue of refining the polynomial degree, but achieved continuously
and at fixed memory cost.

\begin{figure}[htbp]
  \centering
  \begin{subfigure}[b]{0.49\textwidth}
    \centering
    \begin{tikzpicture}
      \begin{axis}[width=\textwidth, height=5.0cm, xlabel={$x$},
        ylabel={$\Psi_n^{(0,1)}(x)$}, domain=-4:4, samples=160,
        legend style={font=\scriptsize, draw=none, at={(0.02,0.02)},
        anchor=south west, legend columns=2}, grid=both, grid style={gray!18},
        tick label style={font=\scriptsize}, label style={font=\small},
        every axis plot/.append style={line width=1pt}]
        \addplot[black] {pi^(-0.25)*exp(-x^2/2)};
        \addlegendentry{$n=0$}
        \addplot[red!75!black] {sqrt(2)*pi^(-0.25)*x*exp(-x^2/2)};
        \addlegendentry{$n=1$}
        \addplot[blue!70!black] {(8*sqrt(pi))^(-0.5)*(4*x^2-2)*exp(-x^2/2)};
        \addlegendentry{$n=2$}
        \addplot[green!55!black] {(48*sqrt(pi))^(-0.5)*(8*x^3-12*x)*exp(-x^2/2)};
        \addlegendentry{$n=3$}
      \end{axis}
    \end{tikzpicture}
    \caption{Fractional Hermite functions of increasing order at $\gamma=1$.}
  \end{subfigure}\hfill
  \begin{subfigure}[b]{0.49\textwidth}
    \centering
    \begin{tikzpicture}
      \begin{axis}[width=\textwidth, height=5.0cm, xlabel={$x$},
        ylabel={$\Psi_2^{(0,\gamma)}(x)$}, domain=-4:4, samples=200,
        legend style={font=\scriptsize, draw=none, at={(0.98,0.02)},
        anchor=south east}, grid=both, grid style={gray!18},
        tick label style={font=\scriptsize}, label style={font=\small},
        every axis plot/.append style={line width=1pt}]
        \addplot[blue!70!black] {(8*sqrt(pi))^(-0.5)*(4*abs(x)^(2*1.0)-2)*exp(-abs(x)^(2*1.0)/2)};
        \addlegendentry{$\gamma=1.0$}
        \addplot[orange!85!black] {(8*sqrt(pi))^(-0.5)*(4*abs(x)^(2*0.7)-2)*exp(-abs(x)^(2*0.7)/2)};
        \addlegendentry{$\gamma=0.7$}
        \addplot[red!75!black] {(8*sqrt(pi))^(-0.5)*(4*abs(x)^(2*0.4)-2)*exp(-abs(x)^(2*0.4)/2)};
        \addlegendentry{$\gamma=0.4$}
      \end{axis}
    \end{tikzpicture}
    \caption{Effect of the fractional order $\gamma$ on a single basis function.}
  \end{subfigure}
  \caption{Fractional Hermite functions defined in \eqref{eq:frac-hermite}.
  They are bounded, decay at infinity, and stay defined on the whole real line
  for every admissible $\gamma$, so they can be evaluated directly on the
  unbounded pre-activation of a linear layer.}
  \label{fig:frac-hermite}
\end{figure}
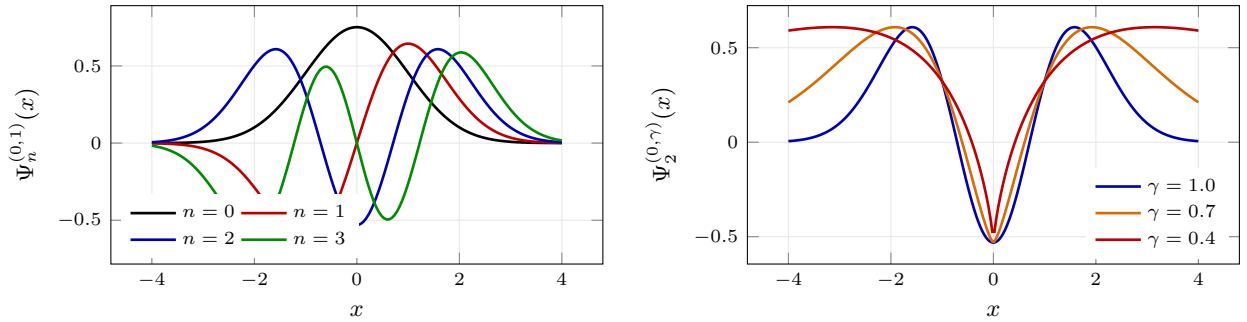

\subsection{The fractional Hermite neural block}
We turn \eqref{eq:frac-hermite} into a trainable module by exposing the shape
$\mu$, the fractional order $\gamma$, and a scale $\lambda$ as learnable
parameters. Each parameter has a natural admissible range that we enforce with a
smooth reparametrization, so that unconstrained gradients can be used.

The shape parameter must satisfy $\mu>-\tfrac12$ for
Theorem~\ref{thm:orth} to hold, and $\mu\ge0$ for the uniform bound of Theorem~\ref{thm:onb}.
We therefore set $\hat\mu=\softplus(\mu)$, which maps $\R$ onto $(0,\infty)$,
so that every trainable value of the shape parameter yields an orthonormal
and uniformly bounded basis. The classical value $\hat\mu=0$ is approached
in the limit $\mu\to-\infty$. The fractional
order must be positive; to avoid the ill-conditioning that accompanies large
exponents we restrict it to the unit interval through
$\hat\gamma=\sigma(\gamma)\in(0,1)$, where $\sigma$ is the logistic sigmoid. This
endows the block with the same ``liquid'' behavior reported for fractional
spectral methods: the network is free to select any fractional order between $0$
and $1$, interpolating continuously between coarse and fine spectral content.
Finally, the scale $\hat\lambda=\softplus(\lambda)>0$ dilates the input and thus
controls the effective receptive field of the Gaussian envelope. Collecting
these definitions, the degree-$q$ \emph{fractional Hermite neural block} (fHNB)
computes the edge function
\begin{equation}
  \phi_q(\zeta) \;=\; \Psi_q^{(\hat\mu,\hat\gamma)}\!\big(\hat\lambda\,\zeta\big)
  \;=\; \psi_q^{(\hat\mu)}\!\Big(\sgn(\hat\lambda\zeta)\,|\hat\lambda\zeta|^{\hat\gamma}\Big),
  \qquad
  \begin{aligned}
    \hat\mu &= \softplus(\mu),\\
    \hat\gamma &= \sigma(\gamma),\\
    \hat\lambda &= \softplus(\lambda).
  \end{aligned}
  \label{eq:fhnb}
\end{equation}
A schematic of the block is given in Figure~\ref{fig:fhnb}.

\begin{remark}[Domain of the block]\label{rem:no-squash}
The argument of
$\psi_q^{(\hat\mu)}$ in \eqref{eq:fhnb} ranges over $\R$, and since
$\hat\mu=\softplus(\mu)>0$, Theorem~\ref{thm:onb} supplies a uniform bound
on $\R$ for every value of the trainable parameters. The trainable scale
$\hat\lambda$ sets the length scale of the Gaussian envelope and is optimized
with the remaining parameters.
\end{remark}

\begin{figure}[htbp]
  \centering
  \begin{tikzpicture}[
    >=Latex, node distance=8mm,
    box/.style={draw, rounded corners, minimum height=8mm, minimum width=15mm,
      align=center, font=\small, fill=blue!5},
    par/.style={draw, circle, minimum size=8mm, font=\small, fill=orange!15},
    op/.style={draw, rounded corners, minimum height=8mm, minimum width=12mm,
      align=center, font=\small, fill=green!6}]
    \node (in) {$\zeta$};
    \node[box, right=10mm of in] (scale) {scale\\$\hat\lambda\,\zeta$};
    \node[box, right=of scale] (frac) {fractional map\\$\sgn(\cdot)\,|\cdot|^{\hat\gamma}$};
    \node[op, right=of frac] (herm) {$\psi_q^{(\hat\mu)}(\cdot)$};
    \node[right=10mm of herm] (out) {$\phi_q(\zeta)$};

    \draw[->] (in) -- (scale);
    \draw[->] (scale) -- (frac);
    \draw[->] (frac) -- (herm);
    \draw[->] (herm) -- (out);

    \node[par, above=12mm of scale] (lam) {$\lambda$};
    \node[par, above=12mm of frac] (gam) {$\gamma$};
    \node[par, above=12mm of herm] (mu) {$\mu$};
    \draw[->, dashed] (lam) -- node[right, font=\scriptsize]{$\softplus$} (scale);
    \draw[->, dashed] (gam) -- node[right, font=\scriptsize]{$\sigma$} (frac);
    \draw[->, dashed] (mu)  -- node[right, font=\scriptsize]{$\softplus$} (herm);

    \begin{scope}[on background layer]
      \node[draw, dashed, rounded corners, fit=(scale)(frac)(herm)(lam)(gam)(mu),
        inner sep=4mm, label={[font=\small\itshape]below:fractional Hermite neural block (fHNB)}] {};
    \end{scope}
  \end{tikzpicture}
  \caption{The fractional Hermite neural block defined in \eqref{eq:fhnb}. A
  pre-activation $\zeta$ is rescaled by $\hat\lambda$, passed through the odd
  fractional map of order $\hat\gamma$, and expanded on the generalized Hermite
  function of order $q$ and shape $\hat\mu$. The three parameters $\lambda$,
  $\gamma$, and $\mu$ (dashed) are trained end-to-end through their smooth
  reparametrizations.}
  \label{fig:fhnb}
\end{figure}
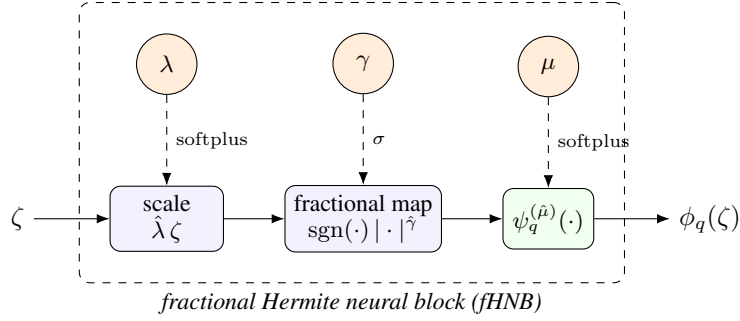

\subsection{From the block to the network}
Following the learnable-activation interpretation of KANs, we expand each
univariate edge function on the fractional Hermite basis and learn the
expansion coefficients. Because Hermite functions are global---non-zero almost
everywhere, with only finitely many zeros---a single basis function already
covers the entire domain, so the outer summation of \eqref{eq:kart} can be
absorbed into a linear combiner $\Phi(\cdot)$ with trainable weights $\theta$.
A degree-$Q$ FHKAN layer mapping $\R^{d}\to\R$ thus reads
\begin{equation}
  \hat{\chi}(\boldsymbol\zeta)
  = \sum_{q=0}^{Q}\sum_{p=1}^{d}\theta_{p,q}\,
    \Psi_q^{(\mu_q,\gamma_q)}\!\big(\hat\lambda_q\,\zeta_p\big),
  \label{eq:fhkan-layer}
\end{equation}
where each order $q$ owns its own triple $(\mu_q,\gamma_q,\lambda_q)$, allowing
different orders to specialize to different scales and smoothness levels.
The
inner expansion plays the role of a Taylor-like series whose terms are
fractional Hermite functions; increasing $Q$ enriches the
representable function class, while the trainable fractional orders let a small
$Q$ already capture non-smooth behavior.

\begin{remark}[Orthogonality in the layer]\label{rem:non-orth}
Theorem~\ref{thm:frac-orth} concerns a single pair $(\mu,\gamma)$ shared by
all orders. In \eqref{eq:fhkan-layer} the functions
$\Psi_q^{(\hat\mu_q,\hat\gamma_q)}(\hat\lambda_q\,\cdot)$, $q=0,\dots,Q$,
carry different shapes, fractional orders, and scales, so they are in general
not mutually orthogonal, and the layer is a nonlinear, adaptive expansion.
Its expressive power is therefore assessed empirically in
Section~\ref{sec:experiments}, and the classical spectral rates of
Section~\ref{sec:hermite-spectral} serve as motivation for the design.
\end{remark}

The triples $(\mu_q,\gamma_q,\lambda_q)$---the shape, fractional order,
and scaling of each block---are trainable parameters. They are updated by
gradient descent through the smooth reparametrizations of \eqref{eq:fhnb},
together with the weights and biases of the linear layers. The quantities
fixed before training are the spectral degree $Q$, the number of layers, the
number of neurons per layer, and the optimizer settings; in this work they
were selected by a small grid search on validation data. Tuning the same
quantities with metaheuristic algorithms, as done for orthogonal basis
functions in \cite{aghaei2025hyperparameter}, is left for future work.

We employ two topologies, depicted in Figure~\ref{fig:architectures}. In the
\emph{sequential} variant (Figure~\ref{fig:fhkan-seq}) a single fHNB of a chosen
order is inserted after each linear layer, and the order may differ from layer
to layer, so that the network gradually composes representations of increasing
complexity. In the \emph{parallel} variant (Figure~\ref{fig:fhkan-par}), used for
the more demanding physics-informed problems, several fHNBs of different orders
process the same hidden representation simultaneously; their outputs are fused by
concatenation (or, optionally, by an attention layer) before a final linear
read-out. The parallel design guarantees that a spectrum of fractional orders
contributes to the approximation, which we found beneficial when solving
differential equations.

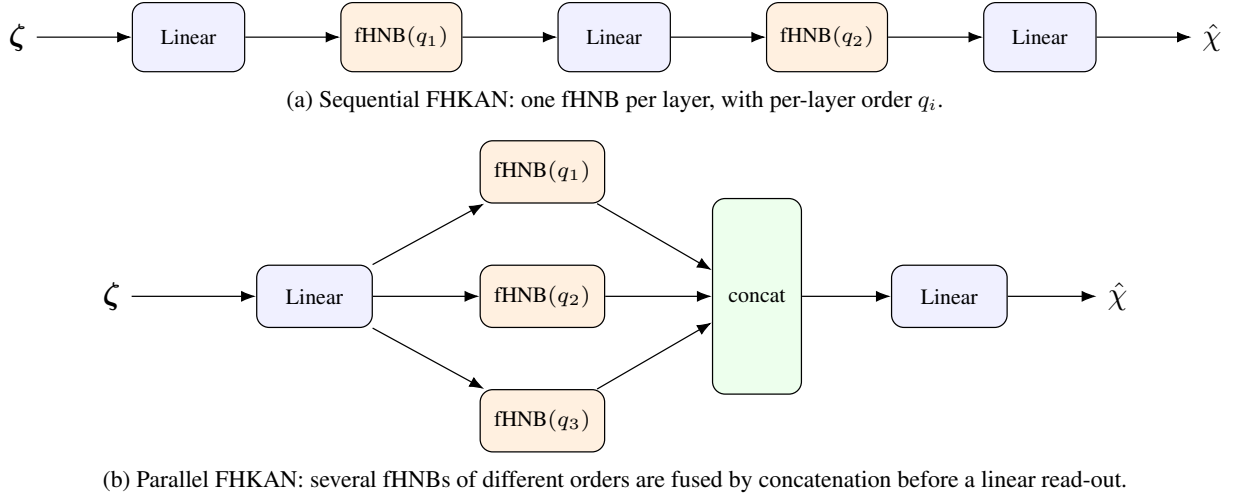
\begin{figure}[htbp]
  \centering
  \begin{subfigure}[b]{\textwidth}
    \centering
    \resizebox{\textwidth}{!}{
    \begin{tikzpicture}[
      >=Latex, node distance=7mm and 11mm,
      lin/.style={draw, rounded corners, minimum height=8mm, minimum width=13mm,
        align=center, font=\scriptsize, fill=blue!6},
      blk/.style={draw, rounded corners, minimum height=8mm, minimum width=14mm,
        align=center, font=\scriptsize, fill=orange!12}]
      \node (in) {$\boldsymbol\zeta$};
      \node[lin, right=of in] (l1) {Linear};
      \node[blk, right=of l1] (b1) {fHNB$(q_1)$};
      \node[lin, right=of b1] (l2) {Linear};
      \node[blk, right=of l2] (b2) {fHNB$(q_2)$};
      \node[lin, right=of b2] (l3) {Linear};
      \node[right=of l3] (out) {$\hat\chi$};
      \foreach \a/\b in {in/l1,l1/b1,b1/l2,l2/b2,b2/l3,l3/out}
        \draw[->] (\a) -- (\b);
    \end{tikzpicture}
    } 
    \caption{Sequential FHKAN: one fHNB per layer, with per-layer order $q_i$.}
    \label{fig:fhkan-seq}
  \end{subfigure}\\[3mm]
  \begin{subfigure}[b]{\textwidth}
    \centering
    \resizebox{0.85\textwidth}{!}{
    \begin{tikzpicture}[
      >=Latex, node distance=5mm and 10mm,
      lin/.style={draw, rounded corners, minimum height=7mm, minimum width=13mm,
        align=center, font=\scriptsize, fill=blue!6},
      blk/.style={draw, rounded corners, minimum height=7mm, minimum width=14mm,
        align=center, font=\scriptsize, fill=orange!12},
      cat/.style={draw, rounded corners, minimum height=22mm, minimum width=10mm,
        align=center, font=\scriptsize, fill=green!7}]
      \node (in) {$\boldsymbol\zeta$};
      \node[lin, right=14mm of in] (l) {Linear};
      \node[blk, right=12mm of l, yshift=14mm] (b1) {fHNB$(q_1)$};
      \node[blk, right=12mm of l] (b2) {fHNB$(q_2)$};
      \node[blk, right=12mm of l, yshift=-14mm] (b3) {fHNB$(q_3)$};
      \node[cat, right=12mm of b2] (c) {concat};
      \node[lin, right=10mm of c] (lo) {Linear};
      \node[right=10mm of lo] (out) {$\hat\chi$};
      \draw[->] (in) -- (l);
      \foreach \b in {b1,b2,b3}{\draw[->] (l) -- (\b); \draw[->] (\b) -- (c);}
      \draw[->] (c) -- (lo);
      \draw[->] (lo) -- (out);
    \end{tikzpicture}
    } 
    \caption{Parallel FHKAN: several fHNBs of different orders are fused by
    concatenation before a linear read-out.}
    \label{fig:fhkan-par}
  \end{subfigure}
  \caption{The two FHKAN topologies used in the experiments.}
  \label{fig:architectures}
\end{figure}

\subsection{Derivatives and regularity of the block}
\label{sec:block-regularity}
Write $s=\hat\lambda\zeta$ and $u=\vartheta_{\hat\gamma}(s)$. For $\zeta\neq0$
the map $\zeta\mapsto u$ is smooth with
$\dd u/\dd\zeta=\hat\gamma\hat\lambda|s|^{\hat\gamma-1}$, and the chain rule gives
\begin{align}
  \phi_q'(\zeta) &= \hat\gamma\,\hat\lambda\,|s|^{\hat\gamma-1}\,
     \big(\psi_q^{(\hat\mu)}\big)'(u),
  \label{eq:block-d1}\\
  \phi_q''(\zeta) &= \hat\gamma^{2}\hat\lambda^{2}\,|s|^{2\hat\gamma-2}\,
     \big(\psi_q^{(\hat\mu)}\big)''(u)
   + \hat\gamma(\hat\gamma-1)\,\hat\lambda^{2}\,\sgn(s)\,|s|^{\hat\gamma-2}\,
     \big(\psi_q^{(\hat\mu)}\big)'(u).
  \label{eq:block-d2}
\end{align}
The first derivative of $\psi_q^{(\hat\mu)}$ is given in closed form by
Proposition~\ref{prop:gen-derivative}, and the second follows by applying
\eqref{eq:gen-derivative} to each of its terms together with
$\big(\psi_n^{(\hat\mu)}/u\big)'=\big(\psi_n^{(\hat\mu)}\big)'/u-\psi_n^{(\hat\mu)}/u^{2}$.
Hence, away from the origin, every derivative of the block is a finite
combination of generalized Hermite functions of neighboring orders multiplied
by powers of $|s|$ and $u^{-1}$. The factor $|s|^{\hat\gamma-1}$ diverges at
$s=0$ because $\hat\gamma<1$, and the behavior of the block at the origin is
described by the following result.

\begin{theorem}[Regularity of the block at the origin]\label{thm:regularity}
Let $\hat\mu\ge0$, $0<\hat\gamma<1$, $\hat\lambda>0$, and define the leading
exponent
\begin{equation*}
  p_q=\begin{cases}
    \hat\gamma\,\hat\mu, & q \text{ even},\ \hat\mu>0,\\
    2\hat\gamma, & q \text{ even},\ \hat\mu=0,\\
    \hat\gamma\,(\hat\mu+1), & q \text{ odd}.
  \end{cases}
\end{equation*}
Then the edge function $\phi_q$ of \eqref{eq:fhnb} satisfies the following.
\begin{enumerate}
  \item[(i)] $\phi_q\in C^{\infty}(\R\setminus\{0\})$, and $\phi_q$ is
  continuous on $\R$.
  \item[(ii)] For every integer $k\ge1$ with $k<p_q$, $\phi_q\in C^{k}(\R)$.
  \item[(iii)] If $p_q$ is not an integer, then for every integer $k>p_q$ the
  $k$-th derivative is unbounded near the origin:
  $|\phi_q^{(k)}(\zeta)|\to\infty$ as $\zeta\to0$.
\end{enumerate}
\end{theorem}
\begin{proof}
Away from the origin $\vartheta_{\hat\gamma}$ and $|x|^{\hat\mu}$ are smooth,
which gives (i) together with the continuity of the composition. By the
Laguerre representation of Section~\ref{sec:hermite}, an even-order function
can be written as $\psi_q^{(\hat\mu)}(u)=|u|^{\hat\mu}E_q(u^{2})$ and an
odd-order function as $\psi_q^{(\hat\mu)}(u)=|u|^{\hat\mu}u\,O_q(u^{2})$, with
entire functions $E_q,O_q$ satisfying $E_q(0)\neq0$ and $O_q(0)\neq0$, because
$L_k^{(\alpha)}(0)=\binom{k+\alpha}{k}>0$ for $\alpha>-1$. Since
$|\vartheta_{\hat\gamma}(s)|=|s|^{\hat\gamma}$ and
$\sgn\vartheta_{\hat\gamma}(s)=\sgn s$,
\begin{equation*}
  \phi_q(\zeta)=
  \begin{cases}
    |s|^{\hat\gamma\hat\mu}\,E_q\big(|s|^{2\hat\gamma}\big), & q \text{ even},\\
    \sgn(s)\,|s|^{\hat\gamma(\hat\mu+1)}\,O_q\big(|s|^{2\hat\gamma}\big), & q \text{ odd}.
  \end{cases}
\end{equation*}
Expanding $E_q$ and $O_q$ in power series represents $\phi_q$ near the origin
as a convergent series of terms $c_j\,\sgn(s)^{q}|s|^{e_j}$ with increasing
exponents $e_j$. If $q$ is odd or $\hat\mu>0$, the leading exponent is
$e_0=p_q$ and its coefficient is $c_0\neq0$. If $q$ is even and $\hat\mu=0$,
the series starts with the constant $E_q(0)$, the next exponent is
$p_q=2\hat\gamma$, and its coefficient equals
$\tfrac12\big(\psi_q^{(0)}\big)''(0)=-\tfrac12(2q+1)\psi_q^{(0)}(0)\neq0$
by the Hermite differential equation. In both cases every non-constant term
has an exponent $e_j\ge p_q$. Every term with $e_j>k$ is $k$ times
continuously differentiable with vanishing derivatives at the origin, which
proves (ii). If $p_q$ is not an integer, the $k$-th derivative of the leading
term equals a nonzero multiple of $|s|^{p_q-k}$, which dominates the remaining
terms as $s\to0$ and diverges for $k>p_q$; this proves (iii).
\end{proof}

Theorem~\ref{thm:regularity} makes the regularity requirements of the
physics-informed residuals explicit. Because $\hat\gamma<1$, a first
derivative that is continuous at the origin requires
$\hat\mu>1/\hat\gamma-1$ for odd orders and $\hat\mu>1/\hat\gamma$ for even
orders, and a continuous second derivative, as required by the Lane--Emden
and Burgers residuals, requires $\hat\mu>2/\hat\gamma-1$ and
$\hat\mu>2/\hat\gamma$, respectively.

\paragraph{Singular set of the network}
In an FHKAN the block acts on the pre-activations
$z_i(\boldsymbol\zeta)$ of the preceding linear layer. The network output is
therefore infinitely differentiable outside the singular set
$\mathcal{Z}=\{\boldsymbol\zeta:\ z_i(\boldsymbol\zeta)=0 \text{ for some block input } i\}$,
and its derivatives there are given exactly by \eqref{eq:block-d1},
\eqref{eq:block-d2} and the chain rule. In the first layer each
$z_i$ is affine, so its zero set is a hyperplane. In deeper layers each $z_i$
is real-analytic on every connected component of the complement of the
singular sets of the preceding layers, because $|x|^{\hat\gamma}$ and
$|x|^{\hat\mu}$ are real-analytic for $x\neq0$. Consequently $\mathcal{Z}$ is
a closed set of Lebesgue measure zero whenever no pre-activation vanishes
identically on an open set.

\paragraph{Treatment of singular points during training}
The physics-informed loss evaluates the residual only at the collocation
nodes. At every node outside $\mathcal{Z}$ the derivatives entering the
residual are finite and are the exact values \eqref{eq:block-d1} and
\eqref{eq:block-d2}, which automatic differentiation reproduces. For a
fixed node, each condition $z_i=0$ defines a set of parameter values of
measure zero, so a node lies exactly on $\mathcal{Z}$ only for exceptional
parameter values. To keep
floating-point evaluation finite at an exactly vanishing pre-activation, the
fractional map is evaluated as $\sgn(x)\max(|x|,\varepsilon)^{\hat\gamma}$, with
$\varepsilon$ equal to the machine epsilon of the floating-point format. This
modification alters $\vartheta_{\hat\gamma}$ only on $(-\varepsilon,\varepsilon)$
and leaves the computed value and derivatives unchanged at every node whose
pre-activations exceed $\varepsilon$ in magnitude. A node at distance $\delta$
from a point of $\mathcal{Z}$ contributes terms of size $O(\delta^{p_q-2})$ to
a second-order residual, so the residual loss penalizes parameter
configurations in which a zero crossing approaches a collocation node. The
residual is thereby enforced pointwise at the nodes.

\subsection{Training considerations}
All parameters---the linear weights $\theta$, biases, and the per-order triples
$(\mu_q,\gamma_q,\lambda_q)$---are optimized jointly by gradient descent. The
forward pass evaluates the recurrence \eqref{eq:hermite-recurrence} up to order
$Q$ and applies the envelope \eqref{eq:hermite-functions}; both steps are
vectorized and differentiable, so gradients flow through the reparametrizations
in \eqref{eq:fhnb} automatically. For physics-informed objectives, the spatial
derivatives required by the differential operator are obtained, at every
collocation node outside the singular set $\mathcal{Z}$, either by
automatic differentiation or, equivalently, in closed form through \eqref{eq:block-d1}
and \eqref{eq:block-d2}, as described in Section~\ref{sec:block-regularity}. Because the basis is bounded, the forward pass uses the raw pre-activations,
and the training and inference behavior stay identical.

This residual viewpoint is the same one that underlies Hermite function
collocation \cite{parand2010hermite}. There, a trial solution is expanded in
the mapped functions $\widehat\psi_n$, the differential residual is formed by
substituting the expansion into the governing operator, and the unknown
coefficients are fixed by requiring the residual to vanish at the mapped
Hermite--Gauss nodes---a nonlinear algebraic system, typically solved by
Newton iteration. Essential data at $\zeta=0$ are built into the trial
function by a lifting of the form $A+B\zeta+\zeta\sum a_n\widehat\psi_n(\zeta)$.
FHKAN keeps the residual, the mapped basis, and the Gauss nodes. The
coefficients, together with the shape, fractional order, and scale of each
block, are obtained by minimizing
$\mathfrak{R}^{\top}\mathfrak{R}+\mathfrak{B}^2+\mathfrak{I}^2$ by gradient
descent. Initial and boundary observations enter through the quadratic
penalties $\mathfrak{B}$ and $\mathfrak{I}$. The physics-informed FHKAN is a
trainable, fractional, multi-block form of this Hermite collocation scheme.
The convergence theory of classical Hermite collocation is formulated
for a fixed orthogonal basis (Section~\ref{sec:hermite-spectral}); for the
trainable, non-orthogonal FHKAN expansion the accuracy of the collocation
solutions is measured numerically in Section~\ref{sec:experiments}.

\section{Experiments}
\label{sec:experiments}
We evaluate FHKAN on data-driven and physics-informed tasks. The data-driven
suite comprises a synthetic regression problem, MNIST digit classification,
Fashion-MNIST denoising, and IMDB sentiment analysis; the physics-informed suite
comprises the Lane--Emden ordinary differential equation, the one-dimensional
Burgers partial differential equation, and a fractional delay differential
equation. The deep-learning models and the physics-informed models were implemented in in PyTorch
\cite{paszke2019pytorch}. Experiments were run on a workstation with an Intel
Core-i7 CPU, 32\,GB of RAM, and an NVIDIA RTX-class GPU. Unless stated otherwise,
training used the Adam optimizer \cite{kingma2014adam}; the physics-informed
models used short Adam warmup followed by L-BFGS. \cite{liu1989limited}. Every reported value is the mean over
ten independent runs. Where a table entry contains a $\pm$ term, the second
number is the standard deviation across those runs.

\subsection{Data-driven tasks}

\subsubsection{Synthetic regression}
\label{sec:regression}
We first approximate the noisy one-dimensional target
\begin{equation*}
  \chi(\zeta) = \sin(\pi\zeta) + 10\,\exp\!\Big(\tfrac{\zeta}{5}\Big) + \tfrac{\epsilon}{100},
  \qquad \epsilon\sim\mathcal{N}(0,1),
\end{equation*}
sampled at $50$ equidistant points on $[-2,1]$ and split into training and test
subsets with a $33\%$ hold-out. A single hidden-layer network with a varying
number of neurons is trained with mean-squared-error loss, Adam at learning rate
$10^{-2}$, at most $500$ iterations, and early stopping with patience $200$.
The additive Gaussian noise $\epsilon$ was generated once, with a fixed
random seed, and the resulting noisy dataset was kept identical across the ten
independent runs. The run-to-run variability reported below comes from the
random initialization of the networks and the stochasticity of the optimizer.
Figure~\ref{fig:reg-pred} compares the fitted curves: FHKAN tracks the ground
truth most faithfully, where ReLU and tanh produce visible bias. Figure~\ref{fig:reg-layers} repeats the
study for multi-layer networks with five neurons per layer and reports the test
mean absolute error as the depth grows; the corresponding statistics are
collected in Table~\ref{tbl:multilayer}. FHKAN attains the lowest error at every
depth, with FHKAN(6) best at one and nine hidden layers and FHKAN(5) best at
three, five, and seven hidden layers, and its standard deviations across the
ten runs lie between $0.007$ and $0.026$.
Table~\ref{tbl:reg-metrics} reports MAE, MSE, RMSE, and the training loss
(the mean squared residual on the training points) for the single-hidden-layer
configuration.

Figure~\ref{fig:params} tracks the trainable parameters of a five-neuron
single-layer FHKAN trained on $250$ samples without early stopping. The shape
$\hat\mu$, the fractional order $\hat\gamma$, and the scale $\hat\lambda$ all
settle to stable values within roughly $250$ iterations, confirming that the
reparametrizations in \eqref{eq:fhnb} lead to a well-behaved optimization
landscape. Finally, Figure~\ref{fig:time} compares the mean execution time of
each activation on a $1000\times1000$ input: the fractional Hermite block is
about $7.9$ times slower than GELU on the CPU but markedly faster than the
spline-based KAN, and the gap narrows substantially on the GPU.

\begin{figure}[htbp]
  \centering
  \begin{subfigure}[b]{0.49\textwidth}
    \centering
    \includegraphics[width=\textwidth, height=5cm]{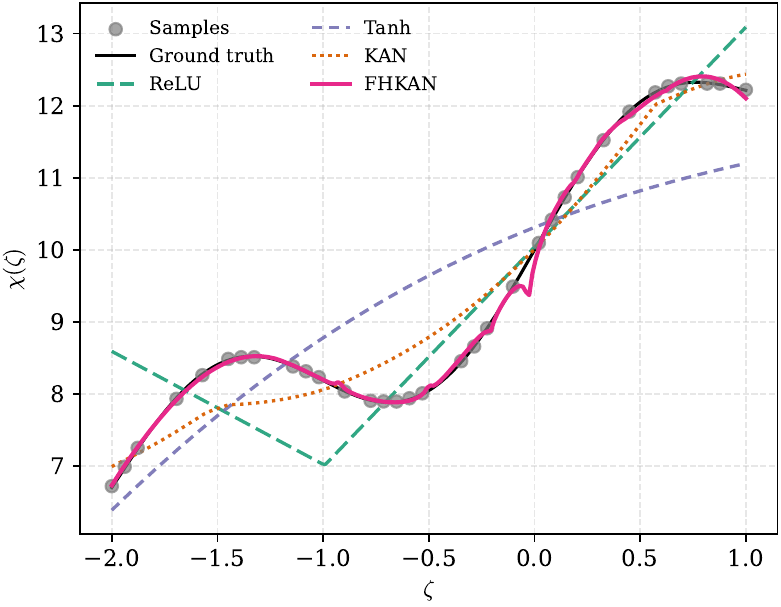}
    \caption{Single-layer fit on the regression target.}
    \label{fig:reg-pred}
  \end{subfigure}\hfill
  \begin{subfigure}[b]{0.49\textwidth}
    \centering
    \includegraphics[width=\textwidth, height=5cm]{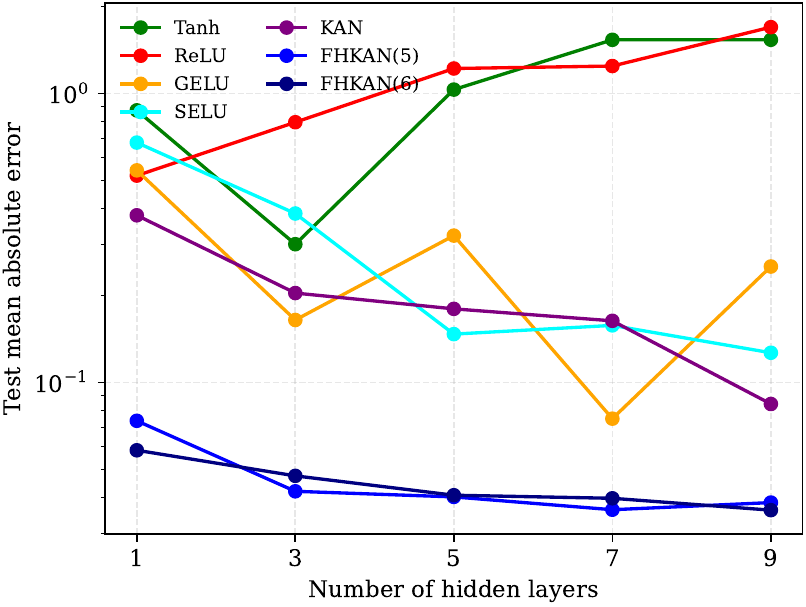}
    \caption{Test error versus network depth.}
    \label{fig:reg-layers}
  \end{subfigure}\\
  \begin{subfigure}[b]{0.49\textwidth}
    \centering
    \includegraphics[width=\textwidth, height=5cm]{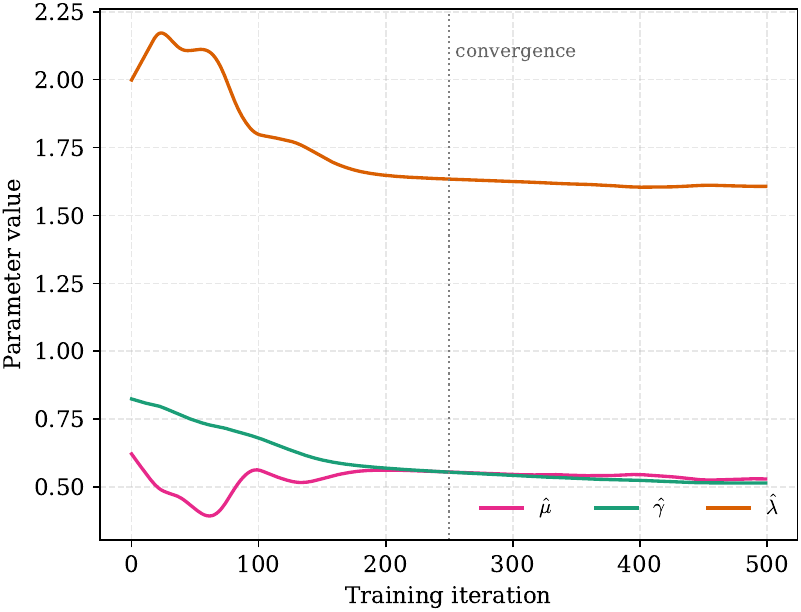}
    \caption{Evolution of $\hat\mu$, $\hat\gamma$, and $\hat\lambda$.}
    \label{fig:params}
  \end{subfigure}\hfill
  \begin{subfigure}[b]{0.49\textwidth}
    \centering
    \includegraphics[width=\textwidth, height=5cm]{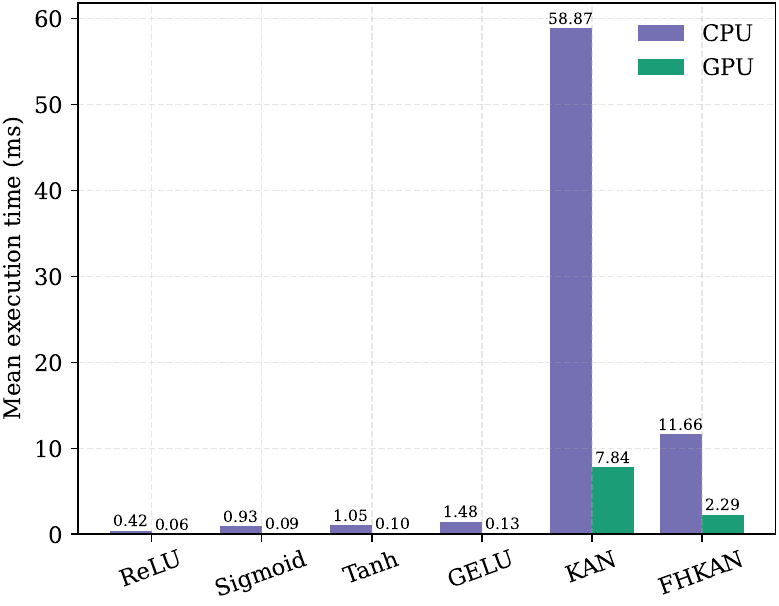}
    \caption{Mean execution time on CPU and GPU.}
    \label{fig:time}
  \end{subfigure}
  \caption{Synthetic regression study. FHKAN is more accurate than standard
  activations and the vanilla KAN, its trainable parameters converge quickly,
  and its run time stays well below that of the spline-based KAN.}
  \label{fig:regression}
\end{figure}

\begin{table}[t]
\centering
\caption{Test MAE across increasing depth for the synthetic regression
experiment.}
\vspace{0.5em}
\small
\setlength{\tabcolsep}{3pt}
\begin{tabularx}{\textwidth}{@{}l*{5}{Y}@{}}
\toprule
\multirow{2}{*}{Act.\ func.} & \multicolumn{5}{c}{Number of hidden layers}\\
\cmidrule(l){2-6}
 & 1 & 3 & 5 & 7 & 9 \\
\midrule
Softplus    & 0.461{\tiny$\pm$0.046} & 0.448{\tiny$\pm$0.031} & 0.448{\tiny$\pm$0.017} & 0.418{\tiny$\pm$0.078} & 0.358{\tiny$\pm$0.107} \\
SiLU        & 0.518{\tiny$\pm$0.125} & 0.388{\tiny$\pm$0.156} & 0.363{\tiny$\pm$0.344} & 0.162{\tiny$\pm$0.128} & 0.272{\tiny$\pm$0.380} \\
SELU        & 0.677{\tiny$\pm$0.087} & 0.384{\tiny$\pm$0.151} & 0.147{\tiny$\pm$0.058} & 0.158{\tiny$\pm$0.075} & 0.127{\tiny$\pm$0.069} \\
Leaky ReLU  & 0.489{\tiny$\pm$0.086} & 0.316{\tiny$\pm$0.140} & 0.214{\tiny$\pm$0.103} & 0.235{\tiny$\pm$0.067} & 0.189{\tiny$\pm$0.093} \\
GELU        & 0.542{\tiny$\pm$0.134} & 0.165{\tiny$\pm$0.066} & 0.322{\tiny$\pm$0.616} & 0.075{\tiny$\pm$0.041} & 0.252{\tiny$\pm$0.618} \\
ELU         & 0.621{\tiny$\pm$0.127} & 0.433{\tiny$\pm$0.171} & 0.384{\tiny$\pm$0.341} & 0.254{\tiny$\pm$0.058} & 0.272{\tiny$\pm$0.221} \\
Tanh        & 0.875{\tiny$\pm$0.012} & 0.301{\tiny$\pm$0.041} & 1.032{\tiny$\pm$0.406} & 1.535{\tiny$\pm$0.001} & 1.535{\tiny$\pm$0.000} \\
Sigmoid     & 0.745{\tiny$\pm$0.082} & 1.332{\tiny$\pm$0.276} & 1.521{\tiny$\pm$0.032} & 1.528{\tiny$\pm$0.018} & 1.531{\tiny$\pm$0.006} \\
ReLU        & 0.520{\tiny$\pm$0.101} & 0.796{\tiny$\pm$1.425} & 1.221{\tiny$\pm$1.914} & 1.245{\tiny$\pm$1.903} & 1.698{\tiny$\pm$2.194} \\
KAN         & 0.379{\tiny$\pm$0.051} & 0.204{\tiny$\pm$0.031} & 0.180{\tiny$\pm$0.066} & 0.163{\tiny$\pm$0.033} & 0.084{\tiny$\pm$0.041} \\
FHKAN(5)    & 0.074{\tiny$\pm$0.026} & \textbf{0.042{\tiny$\pm$0.011}} & \textbf{0.040{\tiny$\pm$0.015}} & \textbf{0.036{\tiny$\pm$0.008}} & 0.038{\tiny$\pm$0.007} \\
FHKAN(6)    & \textbf{0.058{\tiny$\pm$0.017}} & 0.048{\tiny$\pm$0.015} & 0.041{\tiny$\pm$0.011} & 0.040{\tiny$\pm$0.012} & \textbf{0.036{\tiny$\pm$0.010}} \\
\bottomrule
\end{tabularx}
\label{tbl:multilayer}
\end{table}

\begin{table}[t]
\centering
\caption{Single-layer regression metrics for the synthetic regression
experiment.}
\vspace{0.5em}
\begin{tabularx}{0.92\textwidth}{@{}l*{4}{Y}@{}}
\toprule
Act.\ func. & MAE & MSE & RMSE & Loss \\
\midrule
ReLU     & $0.520\pm0.101$ & $0.3425\pm0.1291$ & $0.575\pm0.106$ & $0.4035\pm0.1113$ \\
Tanh     & $0.875\pm0.012$ & $0.9343\pm0.0396$ & $0.966\pm0.021$ & $0.9820\pm0.0471$ \\
KAN      & $0.379\pm0.051$ & $0.0437\pm0.0492$ & $0.433\pm0.049$ & $0.1731\pm0.0387$ \\
FHKAN(5) & $0.074\pm0.026$ & $0.0127\pm0.0078$ & $0.107\pm0.036$ & $0.0038\pm0.0024$ \\
FHKAN(6) & $\mathbf{0.058\pm0.017}$ & $\mathbf{0.0086\pm0.0059}$ & $\mathbf{0.087\pm0.032}$ & $\mathbf{0.0025\pm0.0016}$ \\
\bottomrule
\end{tabularx}
\label{tbl:reg-metrics}
\end{table}

\subsubsection{Image classification}
We next test FHKAN as the activation of a convolutional classifier on the MNIST
handwritten-digit dataset \cite{lecun1998mnist}, which contains $60{,}000$
training and $10{,}000$ test grayscale images of size $28\times28$. We hold out
$6{,}000$ training images for validation, scale the pixels to $[0,1]$, and feed
them to the convolutional network of Figure~\ref{fig:arch-mnist} with a batch
size of $512$. The activation is set to one of sigmoid, hyperbolic tangent, ReLU,
or fHNB of order $q\in\{2,\dots,6\}$, and the network is trained for $30$ epochs
with Adam. Table~\ref{tbl:mnist} and Figure~\ref{fig:mnist-curves} report the
results: the fractional Hermite block delivers the lowest loss and the highest
accuracy, with the order-$4$ block being the strongest configuration.

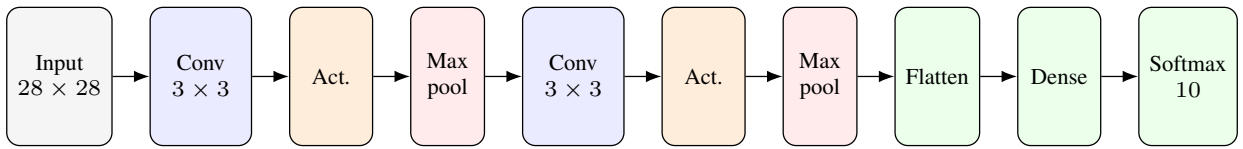
\begin{figure}[htbp]
  \centering
  \resizebox{\textwidth}{!}{
  \begin{tikzpicture}[
    >=Latex, node distance=4mm,
    layer/.style={draw, rounded corners, minimum height=15mm, align=center,
      font=\scriptsize},
    conv/.style={layer, fill=blue!8, minimum width=11mm},
    pool/.style={layer, fill=red!8, minimum width=8mm},
    act/.style={layer, fill=orange!14, minimum width=9mm},
    dense/.style={layer, fill=green!8, minimum width=9mm}]
    \node[layer, fill=gray!8, minimum width=10mm] (in) {Input\\$28\times28$};
    \node[conv, right=of in] (c1) {Conv\\$3\times3$};
    \node[act, right=of c1] (a1) {Act.};
    \node[pool, right=of a1] (p1) {Max\\pool};
    \node[conv, right=of p1] (c2) {Conv\\$3\times3$};
    \node[act, right=of c2] (a2) {Act.};
    \node[pool, right=of a2] (p2) {Max\\pool};
    \node[dense, right=of p2] (fl) {Flatten};
    \node[dense, right=of fl] (d1) {Dense};
    \node[dense, right=of d1] (out) {Softmax\\$10$};
    \foreach \a/\b in {in/c1,c1/a1,a1/p1,p1/c2,c2/a2,a2/p2,p2/fl,fl/d1,d1/out}
      \draw[->] (\a) -- (\b);
  \end{tikzpicture}
  } 
  \caption{Convolutional architecture used for MNIST classification. The
  activation blocks are replaced by sigmoid, tanh, ReLU, or an fHNB.}
  \label{fig:arch-mnist}
\end{figure}

\begin{table}[htbp]
\centering
\caption{Loss and accuracy of a convolutional classifier on MNIST for different
activation functions.}
\vspace{0.5em}
\setlength{\tabcolsep}{12pt}
\begin{tabular}{@{}lcc@{}}
\toprule
Act.\ func. & Loss & Accuracy (\%) \\
\midrule
Sigmoid  & $0.0611\pm0.0031$ & $98.092\pm0.0814$ \\
Tanh     & $0.0322\pm0.0018$ & $98.904\pm0.0572$ \\
ReLU     & $0.0256\pm0.0013$ & $99.140\pm0.0386$ \\
fHNB(2)  & $0.0249\pm0.0016$ & $99.156\pm0.0471$ \\
fHNB(3)  & $0.0227\pm0.0018$ & $99.211\pm0.0762$ \\
fHNB(4)  & $\mathbf{0.0209}\pm0.0012$ & $\mathbf{99.272}\pm0.0508$ \\
fHNB(5)  & $0.0231\pm0.0011$ & $99.218\pm0.0455$ \\
fHNB(6)  & $0.0274\pm0.0026$ & $99.061\pm0.1106$ \\
\bottomrule
\end{tabular}
\label{tbl:mnist}
\end{table}

\begin{figure}[htbp]
  \begin{subfigure}{0.49\textwidth}
    \centering
    \includegraphics[width=\linewidth]{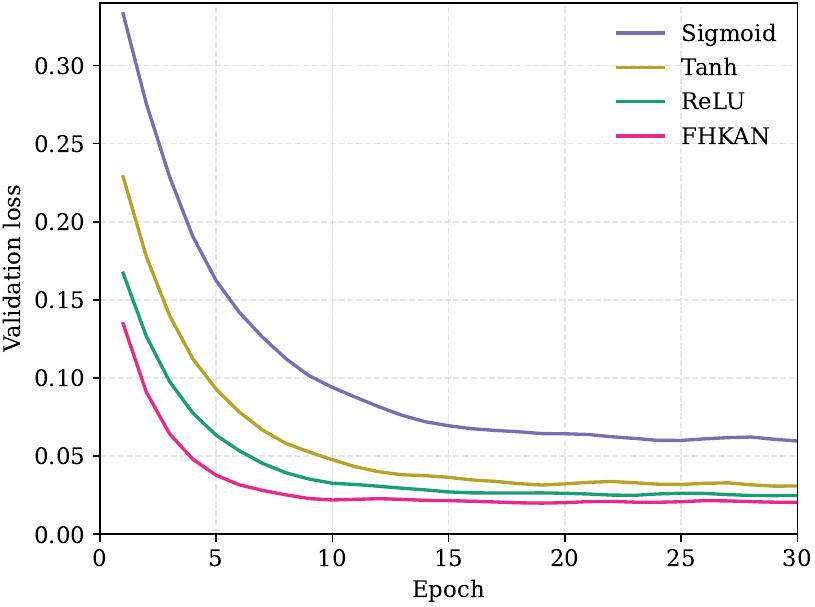}
    \caption{Validation loss.}
  \end{subfigure}\hfill
  \begin{subfigure}{0.49\textwidth}
    \centering
    \includegraphics[width=\linewidth]{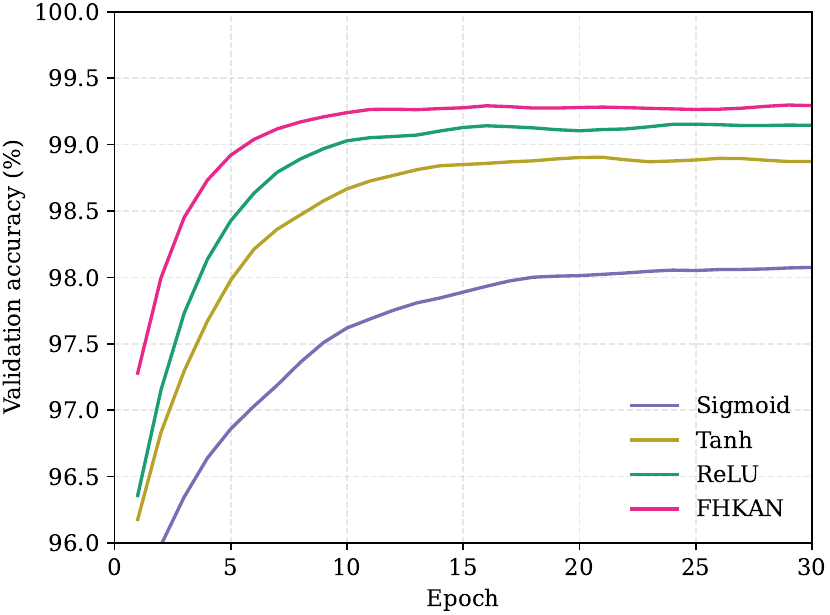}
    \caption{Validation accuracy.}
  \end{subfigure}
  \caption{Training curves on MNIST. The FHKAN curve reaches a lower
  validation loss and a higher validation accuracy than the sigmoid, tanh, and
  ReLU networks from the first epoch onward.}
  \label{fig:mnist-curves}
\end{figure}

\subsubsection{Image denoising}
For denoising we build a convolutional autoencoder, shown in
Figure~\ref{fig:arch-ae}, and benchmark it on Fashion-MNIST \cite{xiao2017fashion},
a harder drop-in replacement for MNIST with $60{,}000$ grayscale images in ten
clothing categories. The autoencoder is first pre-trained to reconstruct clean
images, which lets it learn the data manifold, and is then trained on pairs of
noisy and clean images, where the noisy inputs are obtained by adding Gaussian
noise of factor $0.3$. Both phases use a batch size of $512$, $20$ epochs, and
Adam. Reconstruction quality is measured by MSE, peak
signal-to-noise ratio (PSNR), and structural similarity (SSIM):
\begin{align*}
  \text{PSNR}(\boldsymbol\zeta,\hat{\boldsymbol\zeta})
  &= 10\log_{10}\!\left(\frac{\text{MAX}^2\{\boldsymbol\zeta\}\cdot MN}{\|\boldsymbol\zeta-\hat{\boldsymbol\zeta}\|_F^2}\right),\\[2pt]
  \text{SSIM}(\boldsymbol\zeta,\hat{\boldsymbol\zeta})
  &= \frac{(2\mu_\zeta\mu_{\hat\zeta}+C_1)(2\sigma_{\zeta\hat\zeta}+C_2)}
         {(\mu_\zeta^2+\mu_{\hat\zeta}^2+C_1)(\sigma_\zeta^2+\sigma_{\hat\zeta}^2+C_2)},
\end{align*}
with $\|\cdot\|_F$ the Frobenius norm, $\mu$ and $\sigma$ the relevant means,
variances and covariance, and $C_1,C_2$ small stabilizing constants.
Table~\ref{tbl:ae} shows that the fHNB-based autoencoder yields the highest PSNR
and SSIM in both the pre-training and the denoising phases.

\begin{figure}[htbp]
  \centering
  \resizebox{\textwidth}{!}{
  \begin{tikzpicture}[
    >=Latex, node distance=3.5mm,
    layer/.style={draw, rounded corners, align=center, font=\scriptsize},
    enc/.style={layer, fill=blue!8, minimum height=16mm, minimum width=9mm},
    bott/.style={layer, fill=orange!16, minimum height=9mm, minimum width=9mm},
    dec/.style={layer, fill=green!8, minimum height=16mm, minimum width=9mm}]
    \node[layer, fill=gray!8, minimum height=18mm, minimum width=9mm] (in) {Noisy\\$28\times28$};
    \node[enc, right=of in] (e1) {Conv\\+Act};
    \node[enc, right=of e1] (e2) {Conv\\+Act};
    \node[bott, right=of e2] (b) {Latent};
    \node[dec, right=of b] (d1) {Deconv\\+Act};
    \node[dec, right=of d1] (d2) {Deconv\\+Act};
    \node[layer, fill=gray!8, minimum height=18mm, minimum width=9mm, right=of d2] (out) {Clean\\$28\times28$};
    \foreach \a/\b in {in/e1,e1/e2,e2/b,b/d1,d1/d2,d2/out} \draw[->] (\a) -- (\b);
  \end{tikzpicture}
  } 
  \caption{Convolutional autoencoder used for Fashion-MNIST denoising. The
  activation is selected among sigmoid, tanh, ReLU, or an fHNB.}
  \label{fig:arch-ae}
\end{figure}
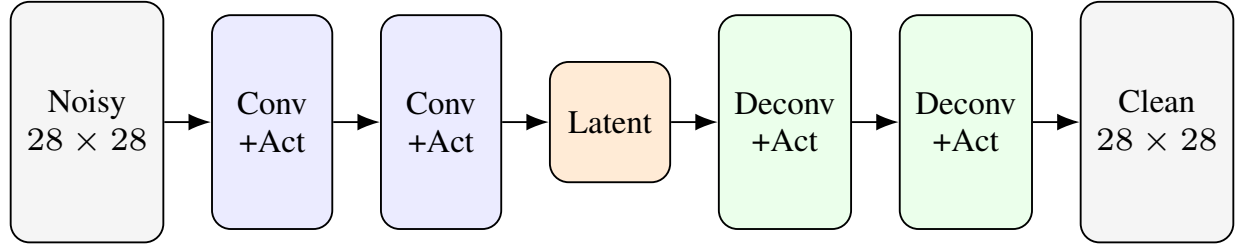

\begin{table}[htbp]
\centering
\caption{Fashion-MNIST denoising quality. Higher PSNR and SSIM and lower MSE are
better; the best mean in each column is shown in bold.}
\vspace{0.5em}
\small
\setlength{\tabcolsep}{3pt}
\begin{tabularx}{\textwidth}{@{}l*{6}{Y}@{}}
\toprule
\multirow{2}{*}{Act.\ func.} & \multicolumn{3}{c}{Pre-train} & \multicolumn{3}{c}{Denoising}\\
\cmidrule(lr){2-4}\cmidrule(lr){5-7}
 & MSE & PSNR & SSIM & MSE & PSNR & SSIM \\
\midrule
Sigmoid & $0.018${\scriptsize$\pm$0.0024} & $18.03${\scriptsize$\pm$0.64} & $0.579${\scriptsize$\pm$0.038} & $0.018${\scriptsize$\pm$0.0016} & $17.81${\scriptsize$\pm$0.29} & $0.564${\scriptsize$\pm$0.016} \\
Tanh    & $0.017${\scriptsize$\pm$0.0017} & $18.08${\scriptsize$\pm$0.48} & $0.583${\scriptsize$\pm$0.029} & $0.018${\scriptsize$\pm$0.0013} & $17.83${\scriptsize$\pm$0.32} & $0.562${\scriptsize$\pm$0.023} \\
ReLU    & $0.017${\scriptsize$\pm$0.0014} & $18.12${\scriptsize$\pm$0.26} & $0.577${\scriptsize$\pm$0.021} & $0.019${\scriptsize$\pm$0.0018} & $17.54${\scriptsize$\pm$0.21} & $0.537${\scriptsize$\pm$0.011} \\
fHNB(5) & $0.013${\scriptsize$\pm$0.0011} & $19.24${\scriptsize$\pm$0.27} & $0.649${\scriptsize$\pm$0.012} & $0.017${\scriptsize$\pm$0.0008} & $18.17${\scriptsize$\pm$0.25} & $0.586${\scriptsize$\pm$0.012} \\
fHNB(6) & $\mathbf{0.012}${\scriptsize$\pm$0.0009} & $\mathbf{19.31}${\scriptsize$\pm$0.28} & $\mathbf{0.653}${\scriptsize$\pm$0.016} & $\mathbf{0.016}${\scriptsize$\pm$0.0012} & $\mathbf{18.24}${\scriptsize$\pm$0.22} & $\mathbf{0.589}${\scriptsize$\pm$0.010} \\
\bottomrule
\end{tabularx}
\label{tbl:ae}
\end{table}

\subsubsection{Sentiment analysis}
The final data-driven task is binary sentiment classification on the IMDB review
corpus \cite{maas2011learning}, with $20{,}000$ training, $5{,}000$ validation,
and $25{,}000$ test reviews. Reviews are cleaned of HTML and punctuation,
tokenized to a vocabulary of $20{,}000$ with a sequence length of $500$, and
classified by the one-dimensional convolutional network of
Figure~\ref{fig:arch-sent}. Training uses Adam, a batch size of $512$, and $10$
epochs. Table~\ref{tbl:sentiment} reports accuracy, precision, recall,
$F_1$-score, and
ROC-AUC. The sigmoid attains the lowest loss. The order-$2$ fHNB attains the
highest recall. The order-$4$ fHNB attains the highest accuracy, precision,
$F_1$-score, and ROC-AUC.
The $F_1$-score is the harmonic mean of precision and recall,
$F_1=2\,\mathrm{Precision}\cdot\mathrm{Recall}/(\mathrm{Precision}+\mathrm{Recall})$,
and summarizes the balance between the two in a single figure.

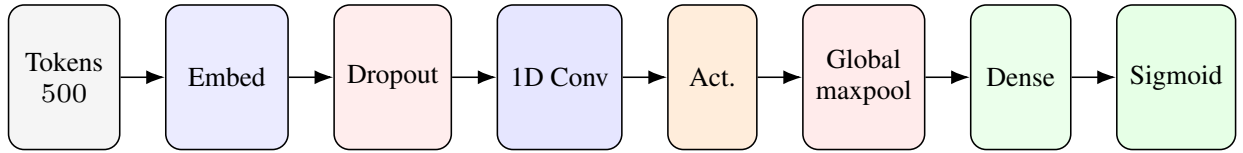
\begin{figure}[htbp]
  \centering
  \resizebox{\textwidth}{!}{
  \begin{tikzpicture}[
    >=Latex, node distance=4mm,
    layer/.style={draw, rounded corners, minimum height=13mm, align=center,
      font=\scriptsize}]
    \node[layer, fill=gray!8, minimum width=10mm] (in) {Tokens\\$500$};
    \node[layer, fill=blue!8, minimum width=11mm, right=of in] (emb) {Embed};
    \node[layer, fill=red!7, minimum width=9mm, right=of emb] (dr) {Dropout};
    \node[layer, fill=blue!10, minimum width=10mm, right=of dr] (c) {1D Conv};
    \node[layer, fill=orange!14, minimum width=8mm, right=of c] (a) {Act.};
    \node[layer, fill=red!8, minimum width=10mm, right=of a] (gp) {Global\\maxpool};
    \node[layer, fill=green!8, minimum width=9mm, right=of gp] (d) {Dense};
    \node[layer, fill=green!10, minimum width=9mm, right=of d] (out) {Sigmoid};
    \foreach \x/\y in {in/emb,emb/dr,dr/c,c/a,a/gp,gp/d,d/out} \draw[->] (\x) -- (\y);
  \end{tikzpicture}
  } 
  \caption{One-dimensional convolutional network used for IMDB sentiment
  analysis. The activation is sigmoid, tanh, ReLU, or an fHNB.}
  \label{fig:arch-sent}
\end{figure}

\begin{table}[htbp]
\centering
\caption{IMDB sentiment-classification metrics. The best mean in each column is
shown in bold.}
\vspace{0.5em}
\scriptsize
\setlength{\tabcolsep}{2.5pt}
\begin{tabularx}{\textwidth}{@{}l*{6}{Y}@{}}
\toprule
Act.\ func. & Loss & Accuracy & Precision & Recall & $F_1$-score & ROC-AUC \\
\midrule
Sigmoid & $\mathbf{0.364}${\tiny$\pm$0.028} & $84.884${\tiny$\pm$0.412} & $84.473${\tiny$\pm$1.274} & $85.480${\tiny$\pm$1.186} & $84.974${\tiny$\pm$0.871} & $92.303${\tiny$\pm$0.361} \\
Tanh    & $0.442${\tiny$\pm$0.047} & $84.672${\tiny$\pm$0.693} & $84.834${\tiny$\pm$2.318} & $84.440${\tiny$\pm$2.147} & $84.636${\tiny$\pm$1.579} & $92.264${\tiny$\pm$0.448} \\
ReLU    & $0.416${\tiny$\pm$0.051} & $86.320${\tiny$\pm$0.974} & $86.554${\tiny$\pm$1.836} & $86.000${\tiny$\pm$1.092} & $86.276${\tiny$\pm$1.065} & $93.366${\tiny$\pm$0.337} \\
fHNB(1) & $0.407${\tiny$\pm$0.058} & $86.988${\tiny$\pm$0.982} & $87.140${\tiny$\pm$2.541} & $86.784${\tiny$\pm$2.463} & $86.961${\tiny$\pm$1.769} & $94.475${\tiny$\pm$0.429} \\
fHNB(2) & $0.379${\tiny$\pm$0.033} & $87.416${\tiny$\pm$0.524} & $87.154${\tiny$\pm$1.128} & $\mathbf{87.768}${\tiny$\pm$0.947} & $87.460${\tiny$\pm$0.737} & $94.681${\tiny$\pm$0.392} \\
fHNB(3) & $0.418${\tiny$\pm$0.041} & $87.156${\tiny$\pm$0.638} & $87.800${\tiny$\pm$1.492} & $86.304${\tiny$\pm$1.173} & $87.046${\tiny$\pm$0.945} & $94.495${\tiny$\pm$0.351} \\
fHNB(4) & $0.420${\tiny$\pm$0.049} & $\mathbf{87.960}${\tiny$\pm$0.671} & $\mathbf{88.590}${\tiny$\pm$2.864} & $87.144${\tiny$\pm$3.142} & $\mathbf{87.861}${\tiny$\pm$2.129} & $\mathbf{94.828}${\tiny$\pm$0.368} \\
fHNB(5) & $0.443${\tiny$\pm$0.081} & $86.104${\tiny$\pm$1.253} & $88.020${\tiny$\pm$2.776} & $83.584${\tiny$\pm$2.091} & $85.745${\tiny$\pm$1.716} & $93.883${\tiny$\pm$0.702} \\
fHNB(6) & $0.454${\tiny$\pm$0.086} & $86.512${\tiny$\pm$1.387} & $87.379${\tiny$\pm$3.264} & $85.352${\tiny$\pm$2.548} & $86.354${\tiny$\pm$2.059} & $94.019${\tiny$\pm$0.814} \\
\bottomrule
\end{tabularx}
\label{tbl:sentiment}
\end{table}

\subsection{Physics-informed tasks}
For physical problems we adopt a physics-informed loss. Given an operator
equation $\mathcal{N}(\chi)=\mathcal{S}$, we define the residual
$\mathfrak{R}(\zeta)=\mathcal{N}(\hat\chi)(\zeta)-\mathcal{S}(\zeta)$ and minimize
\begin{equation*}
  \text{Loss}(\boldsymbol\zeta) = \mathfrak{R}(\boldsymbol\zeta)^{\!\top}\mathfrak{R}(\boldsymbol\zeta)
  + \mathfrak{B}^2 + \mathfrak{I}^2,
\end{equation*}
where $\mathfrak{B}$ and $\mathfrak{I}$ penalize violations of the boundary and
initial conditions and $\boldsymbol\zeta$ collects $N$ collocation points
\cite{raissi2019physics}. Integer-order derivatives are obtained by automatic
differentiation; for the Caputo fractional derivative we use the L1 operational
matrix of \cite{taheri2024accelerating}. All physics-informed models use the
parallel FHKAN of Figure~\ref{fig:fhkan-par}.
The trainable exponent $\hat\gamma$ of the fHNB belongs to the
approximation space. It shapes the spectral content of the basis, and the
sigmoid reparametrization can approach the classical value $\hat\gamma\to 1$.
Fractional Jacobi and Chebyshev functions have likewise been used for
integer-order problems
\cite{parand2017fractional, parand2017accurate, hajimohammadi2021fractional}.
On the Lane--Emden and Burgers problems the differential operator is of
integer order. In Section~\ref{sec:fde} the Caputo order of the equation is
prescribed as $\alpha=0.3$.

\subsubsection{Ordinary differential equations}
\label{sec:ode}
We first solve the standard Lane--Emden equation, a singular second-order
boundary-value problem,
\begin{equation*}
  \frac{\dd^2}{\dd\zeta^2}\chi(\zeta) + \frac{2}{\zeta}\frac{\dd}{\dd\zeta}\chi(\zeta) + \chi^{m}(\zeta) = 0,
  \qquad \chi(0)=1,\quad \chi'(0)=0,
\end{equation*}
for the polytropic indices $m=0,\dots,5$.
Closed-form solutions exist only for $m=0,1,5$,
\begin{equation*}
  \chi_{0}(\zeta)=1-\frac{\zeta^{2}}{6},\qquad
  \chi_{1}(\zeta)=\frac{\sin\zeta}{\zeta},\qquad
  \chi_{5}(\zeta)=\Bigl(1+\frac{\zeta^{2}}{3}\Bigr)^{-1/2},
\end{equation*}
while $m=2,3,4$ have no analytical solution and are compared against the
high-accuracy numerical tables of Horedt \cite{horedt2004polytropes} and
against the collocation residual itself.

The Lane--Emden equation has attracted a wide spectrum of numerical and
analytical solution techniques over the decades. Following the comparative
presentation of \cite{hadian2020single}, Table~\ref{tbl:lane-emden-methods}
compiles the principal families of methods proposed in the literature for this
class of equations, which places the neural approach adopted here in its
broader methodological context.

\begin{table}[htbp]
\centering
\caption{Summary of the principal families of methods proposed in the
literature for solving Lane--Emden type equations, compiled in the spirit of
the comparative tables of \cite{hadian2020single}.}
\vspace{0.5em}
\small
\begin{tabularx}{\textwidth}{@{}lX@{}}
\toprule
Method family & References \\
\midrule
Meshless methods & \cite{parand2016rbf, parand2011improved, parand2017meshless} \\
Finite difference and multistep methods & \cite{hojjati2011multistep, verma2018convergence} \\
Variational and perturbation iteration methods & \cite{wazwaz2015solving, dehghan2008approximate, yildirim2009solutions} \\
Spectral and pseudo-spectral methods & \cite{adibi2010modified, parand2008generalized, parand2010hermite} \\
Analytical and semi-analytical methods & \cite{wazwaz2001new, wazwaz2002new, vangorder2011elegant} \\
Reproducing kernel Hilbert space methods & \cite{syam2018accurate, azarnavid2015picard, hashemi2017solving} \\
Machine learning and neural methods & \cite{mall2014chebyshev, hadian2020single, parand2023neural, mazraeh2024gepinn} \\
\bottomrule
\end{tabularx}
\label{tbl:lane-emden-methods}
\end{table}

The Lane--Emden problem is posed on $[0,+\infty)$. The fHNB is evaluated
at the mapped coordinate $\omega$ of \eqref{eq:dehghan-map}, with scaling
$k=1$; under this map, $\zeta$ near the origin corresponds to large negative
$\omega$. The collocation nodes are the pre-images, under
\eqref{eq:dehghan-map}, of the Hermite--Gauss roots of $H_{N+1}$. Those roots
are the eigenvalues of the tridiagonal matrix built from the three-term
recurrence \eqref{eq:hermite-recurrence} (Golub--Welsch), and the training
locations are the pre-images $\zeta_j>0$.
Figure~\ref{fig:collocation} compares these mapped Hermite nodes with a
uniform grid. The nodes cluster near the origin, where the Lane--Emden
coefficient $2/\zeta$ is singular. We take $N=128$ mapped positive
Hermite--Gauss nodes.
The architecture uses six fHNBs, each attached to a $10$-neuron layer, whose
$60$ outputs feed a single regression neuron.

\begin{figure}[htbp]
  \centering
  \includegraphics[width=0.72\textwidth]{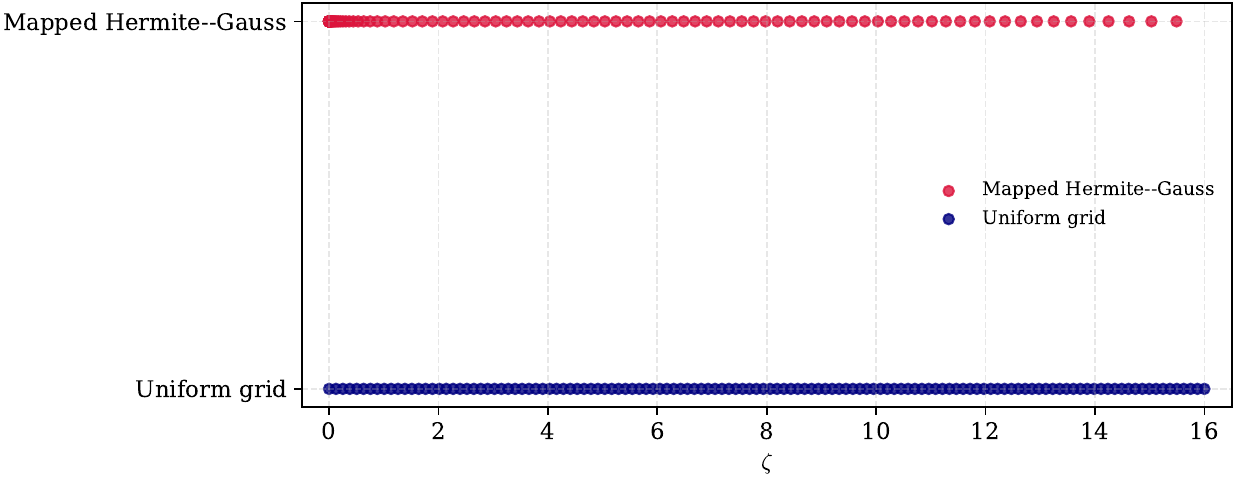}
  \caption{Mapped Hermite--Gauss collocation nodes
  (roots of $H_{N+1}$ pushed through \eqref{eq:dehghan-map}) and a uniform
  grid on a truncated interval $[0,16]$. Under \eqref{eq:dehghan-map} the
  origin is the image of $\omega=-\infty$. The nodes cluster near $\zeta=0$,
  where the coefficient $2/\zeta$ is singular.}
  \label{fig:collocation}
\end{figure}

Figure~\ref{fig:lane-emden} shows the predicted profiles.
For $m=0,1,5$ the solid curves are the exact solutions quoted above; for
$m=2,3,4$, which have no closed form, the solid curves are high-order Runge--Kutta
reference solutions validated against Horedt's tables
\cite{horedt2004polytropes}. In all six cases the computational interval is the
semi-infinite ray $[0,+\infty)$, truncated in the plot at the first positive
root (or at $\zeta=16$ for $m=5$, which has no finite root).
The first root of each profile carries physical meaning: it is the smallest
$\zeta>0$ at which the profile intersects the horizontal axis, i.e., at which
$\chi(\zeta)=0$, and it corresponds to the dimensionless radius of the
polytropic star. Table~\ref{tbl:lane-emden} reports this first root for $m=0,\dots,5$. It also gives the pointwise MAE, MSE, and the collocation residual loss. MAE and MSE are computed against the exact solution for $m=0,1,5$ and against the Horedt/RK reference for $m=2,3,4$.
The $m=5$ polytrope
decays as $\zeta^{-1}$ and never crosses the axis, so the first-root column is
$\infty$; accuracy for that case is measured by the pointwise metrics in the same table.

\begin{figure}[htbp]
  \centering
  \includegraphics[width=0.72\textwidth]{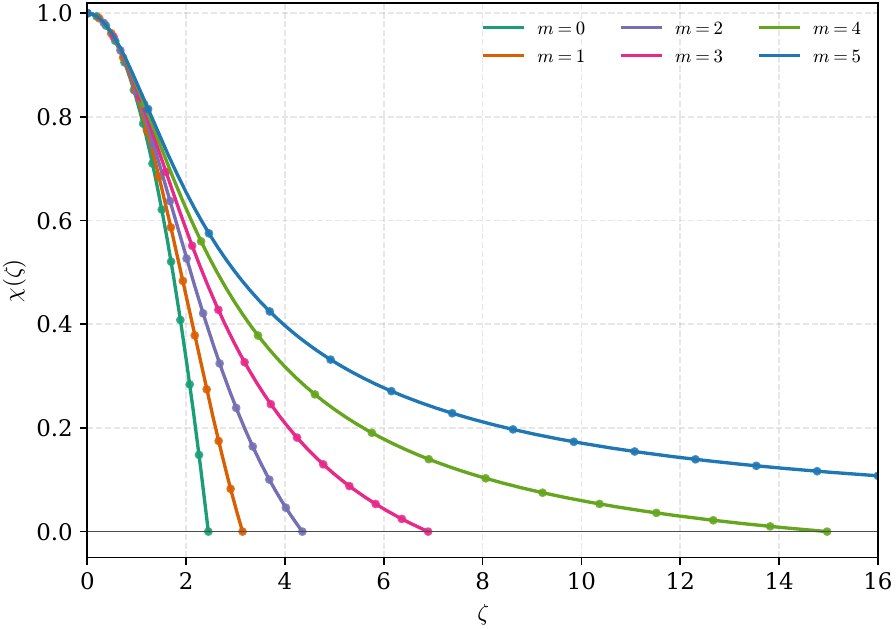}
  \caption{Lane--Emden profiles for $m=0,\dots,5$ obtained with the parallel
  FHKAN (markers) against the reference solutions (solid lines).
  The domain is the semi-infinite interval $[0,+\infty)$. For $m=2,3,4$
  the reference is a high-order numerical solution.}
  \label{fig:lane-emden}
\end{figure}

\begin{table}[htbp]
\centering
\caption{First positive root of the Lane--Emden solution on $[0,+\infty)$,
together with the pointwise mean absolute error, the mean squared error, and the collocation residual loss.
For $m=0,1$ the reference is exact; for $m=2,3,4$ it is Horedt's numerical table
\cite{horedt2004polytropes}; for $m=5$ there is no finite root. MAE and MSE are computed against the exact solution for $m=0,1,5$ and against the Horedt/RK reference for $m=2,3,4$. The loss column is the collocation residual.}
\vspace{0.5em}
\scriptsize
\setlength{\tabcolsep}{2pt}
\renewcommand{\arraystretch}{1.15}
\begin{tabularx}{\textwidth}{@{}c*{6}{Y}@{}}
\toprule
$m$ & \makecell{Reference\\first root} & FHKAN & \makecell{Absolute error} & MAE & MSE & Loss \\
\midrule
0 & \makecell{$2.44948974$\\(exact)} & \makecell{$2.44949712$\\$\pm1.8\times10^{-6}$} & \makecell{$(7.37\pm1.8)$$\times10^{-6}$} & \makecell{$(2.14\pm0.38)$$\times10^{-6}$} & \makecell{$(6.81\pm2.30)$$\times10^{-12}$} & \makecell{$(4.1\pm0.8)$$\times10^{-8}$} \\
1 & \makecell{$3.14159265$\\(exact)} & \makecell{$3.14159083$\\$\pm4.7\times10^{-7}$} & \makecell{$(1.82\pm0.47)$$\times10^{-6}$} & \makecell{$(1.47\pm0.29)$$\times10^{-6}$} & \makecell{$(3.22\pm0.93)$$\times10^{-12}$} & \makecell{$(3.6\pm0.7)$$\times10^{-8}$} \\
2 & \makecell{$4.35287460$\\(Horedt)} & \makecell{$4.35287902$\\$\pm1.1\times10^{-6}$} & \makecell{$(4.42\pm1.1)$$\times10^{-6}$} & \makecell{$(3.85\pm0.81)$$\times10^{-6}$} & \makecell{$(2.41\pm0.86)$$\times10^{-11}$} & \makecell{$(5.8\pm1.2)$$\times10^{-8}$} \\
3 & \makecell{$6.89684862$\\(Horedt)} & \makecell{$6.89684723$\\$\pm3.6\times10^{-7}$} & \makecell{$(1.39\pm0.36)$$\times10^{-6}$} & \makecell{$(2.26\pm0.47)$$\times10^{-6}$} & \makecell{$(8.05\pm3.01)$$\times10^{-12}$} & \makecell{$(6.2\pm1.4)$$\times10^{-8}$} \\
4 & \makecell{$14.9715463$\\(Horedt)} & \makecell{$14.9716277$\\$\pm1.9\times10^{-5}$} & \makecell{$(8.14\pm1.9)$$\times10^{-5}$} & \makecell{$(4.12\pm0.96)$$\times10^{-5}$} & \makecell{$(2.73\pm1.17)$$\times10^{-9}$} & \makecell{$(2.1\pm0.4)$$\times10^{-7}$} \\
5 & \makecell{$\infty$\\(exact)} & \makecell{no finite\\root} & --- & \makecell{$(6.51\pm1.27)$$\times10^{-6}$} & \makecell{$(6.02\pm2.19)$$\times10^{-11}$} & \makecell{$(7.4\pm1.5)$$\times10^{-8}$} \\
\bottomrule
\end{tabularx}
\label{tbl:lane-emden}
\end{table}

\subsubsection{Partial differential equations}
As a partial differential equation we consider the one-dimensional Burgers
equation \cite{khater2008chebyshev}
\begin{equation*}
  \frac{\partial}{\partial\tau}\chi(\zeta,\tau)
  + m_0\,\chi(\zeta,\tau)\frac{\partial}{\partial\zeta}\chi(\zeta,\tau)
  - m_1\frac{\partial^2}{\partial\zeta^2}\chi(\zeta,\tau) = 0,
\end{equation*}
whose traveling-wave solution is
$\chi(\zeta,\tau)=\tfrac{m_2}{m_0}-2\tfrac{m_1}{m_0}\tanh(\zeta-m_2\tau)$.
This profile satisfies the differential equation identically. The spatial
and temporal variables lie in the rectangle
$(\zeta,\tau)\in[0,1]\times[0,1]$: the initial condition is imposed at $\tau=0$
and Dirichlet boundary conditions, both taken from the exact solution, are
imposed at $\zeta=0$ and $\zeta=1$.
The architecture is the parallel FHKAN with two input coordinates.
The fHNB is evaluated at the affine image $\omega=2\zeta-1$, and likewise
for $\tau$. The collocation set is a $100\times 100$ tensor product of mapped
Hermite--Gauss nodes. Formulas for the nodes, weights, and classical ladder derivatives
follow \cite{parand2010hermite,shen2009some,tang1993hermite}.
Figure~\ref{fig:burgers} reports the predicted field and the pointwise residual
for two parameter regimes; the residual stays at the level of $10^{-4}$
throughout the domain in both cases.

\begin{figure}[htbp]
  \centering
  \begin{subfigure}{0.46\textwidth}
    \centering
    \includegraphics[width=\textwidth]{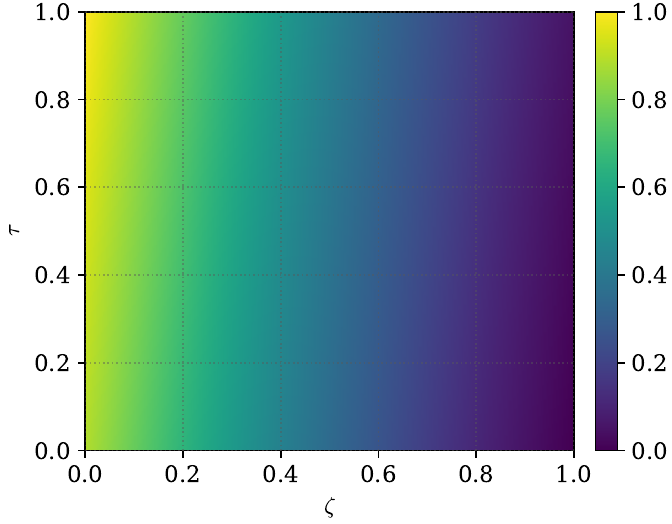}
    \caption{Prediction, regime 1.}
  \end{subfigure}\hfill
  \begin{subfigure}{0.46\textwidth}
    \centering
    \includegraphics[width=\textwidth]{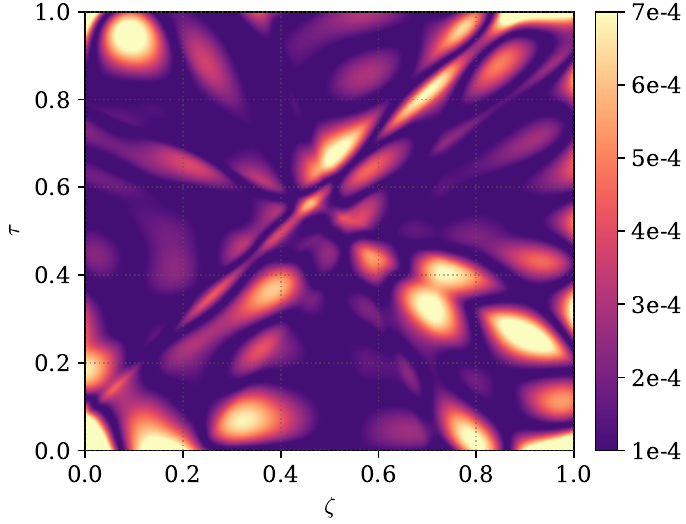}
    \caption{Residual, regime 1.}
  \end{subfigure}\\
  \begin{subfigure}{0.46\textwidth}
    \centering
    \includegraphics[width=\textwidth]{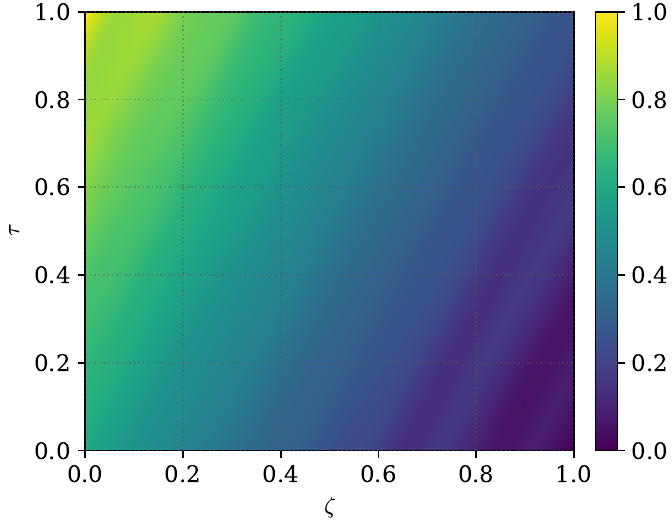}
    \caption{Prediction, regime 2.}
  \end{subfigure}\hfill
  \begin{subfigure}{0.46\textwidth}
    \centering
    \includegraphics[width=\textwidth]{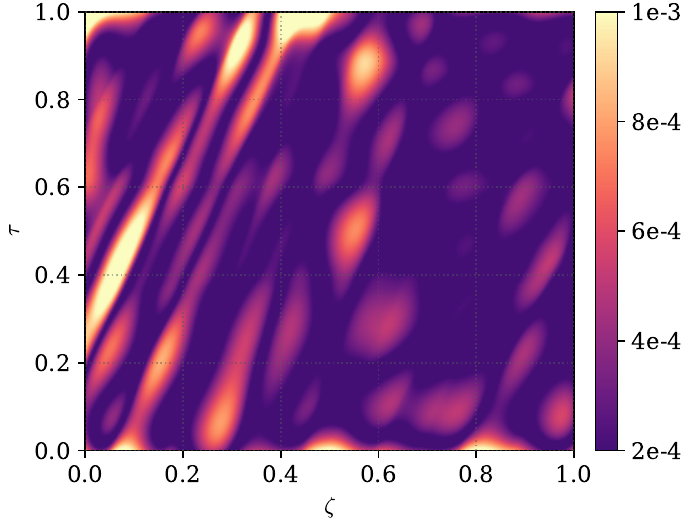}
    \caption{Residual, regime 2.}
  \end{subfigure}
  \caption{Burgers equation solved with the parallel FHKAN for
  $(m_0,m_1,m_2)=(1,0.01,0.1)$ (top) and $(0.1,10^{-4},0.5)$ (bottom).
  The computational domain is $(\zeta,\tau)\in[0,1]\times[0,1]$; collocation
  uses a $100\times 100$ tensor product of mapped Hermite--Gauss nodes.}
  \label{fig:burgers}
\end{figure}

\subsubsection{Fractional differential equations}
\label{sec:fde}
Finally we solve a fractional delay differential equation of Caputo type
\cite{taheri2023bridging},
\begin{equation*}
  \frac{\dd^{0.3}}{\dd\zeta^{0.3}}\chi(\zeta)
  = \chi(\zeta-1) - \chi(\zeta) + 1 - 3\zeta + 3\zeta^2 + \frac{2000\,\zeta^{2.7}}{1071\,\Gamma(0.7)},
  \qquad \chi(0)=0,
\end{equation*}
whose exact solution is $\chi(\zeta)=\zeta^{3}$. The Caputo derivative is
\begin{equation*}
  \frac{\dd^{\alpha}}{\dd\zeta^{\alpha}}\chi(\zeta)
  = \frac{1}{\Gamma(1-\alpha)}\int_0^{\zeta}\frac{\chi'(\tau)}{(\zeta-\tau)^{\alpha}}\,\dd\tau,
  \qquad 0<\alpha<1,
\end{equation*}
and we approximate it with the lower-triangular operational matrix
$\mathcal{D}^{\alpha}$ of \cite{taheri2024accelerating, firoozsalari2023deepfdenet}.
As in the Burgers experiment, the bounded interval is mapped affinely onto
a Hermite window and collocation uses mapped Hermite--Gauss nodes
\cite{parand2010hermite}.
The network stacks six fHNBs of $10$ neurons each, a concatenation layer, a
$60\times10$ projection, three dense layers, and a regression neuron. Figure~\ref{fig:delay} reports the prediction, the pointwise
error against the exact solution, and the residual loss. The fractional Hermite
basis, whose own fractional order is free to match the $0.3$-order dynamics,
attains an absolute error below $10^{-4}$; the mild oscillation visible in the
error is a known by-product of using several parallel blocks of different orders.

\begin{figure}[htbp]
  \centering
  \includegraphics[width=\textwidth]{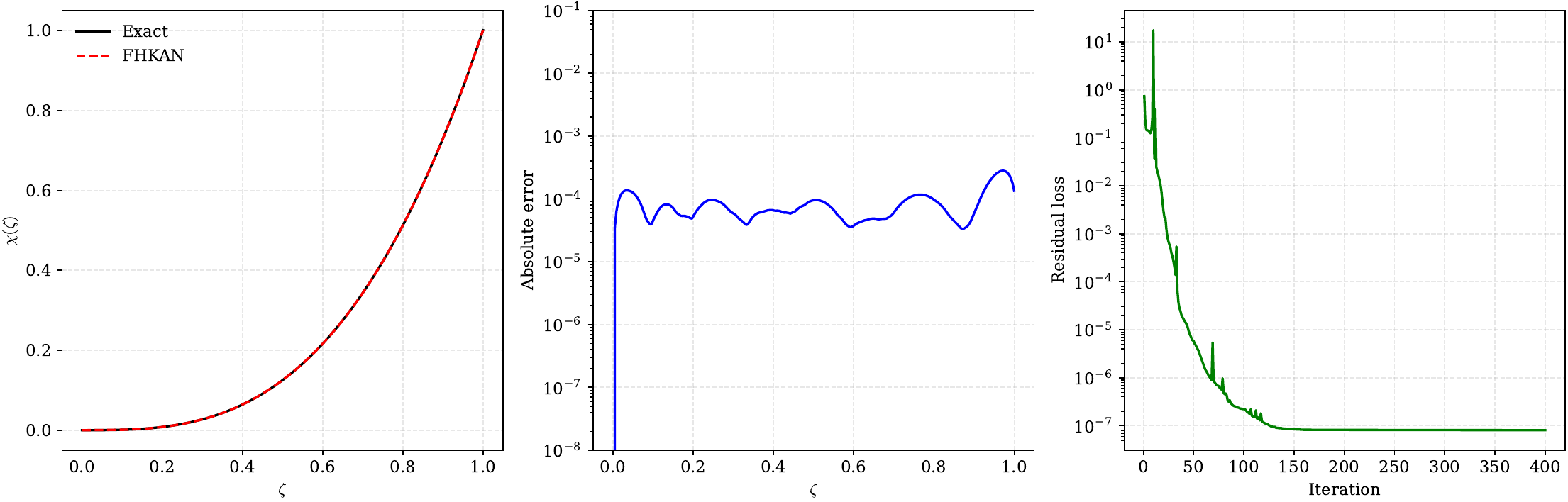}
  \caption{Fractional delay differential equation solved with the parallel
  FHKAN: prediction (left), absolute error against the exact solution (center),
  and residual loss during training (right).}
  \label{fig:delay}
\end{figure}

\FloatBarrier
\section{Parameter efficiency}
\label{sec:discussion}

To compare FHKAN with other KAN architectures, we conducted a
benchmark on the synthetic regression target of
Section~\ref{sec:regression}. A spline KAN \cite{liu2024kan}, a Chebyshev KAN
(ChebyKAN) \cite{ss2024chebyshev}, and FHKAN(5) were each trained with a
single hidden layer of width $h\in\{5,8\}$, which gives six models with $19$
to $224$ trainable parameters. All six models follow the training protocol of
the parameter study in Figure~\ref{fig:params}: $250$ samples of the target on
$[-2,1]$ with the same noise level, and training without early stopping. The
test MAE is measured against the noiseless target. This
longer training budget lets all three architectures fit the target much more
closely than in the depth study of Table~\ref{tbl:multilayer}, whose networks
are trained on $33$ points for at most $500$ iterations with early stopping.
For instance, the spline KAN with $h=5$ attains a test MAE of
$1.01\times10^{-2}$ in this benchmark and $0.379$ in Table~\ref{tbl:multilayer}.
Absolute errors are therefore compared within each of the two experiments.

Figure~\ref{fig:param-efficiency} presents a Pareto analysis of the
results. The plot shows the test MAE against the number of trainable
parameters, with the mean training time encoded by color. ChebyKAN with
$h=8$ achieves the lowest MAE of $2.48\times10^{-3}$ with $73$ parameters, and
ChebyKAN with $h=5$ reaches $2.76\times10^{-3}$ with $46$ parameters; the two
Chebyshev models also have the shortest training times, about $1.9$\,s and
$1.6$\,s. FHKAN(5) with $h=8$ reaches a comparable MAE of $2.98\times10^{-3}$
using only $28$ parameters, about $40\%$ of the parameters of the most
accurate model, at a training time of about $5.3$\,s. With $h=5$, FHKAN(5) is
the smallest model of the benchmark, with $19$ parameters, and attains a MAE
of $1.36\times10^{-2}$ in about $3.7$\,s. The spline KANs are the least
economical: with $140$ and $224$ parameters they reach MAEs of
$1.01\times10^{-2}$ and $1.55\times10^{-2}$ and require about $8.9$\,s and
$16.9$\,s of training. The Pareto frontier, formed by the models for which no
other model is simultaneously more accurate and smaller, consists of the two
FHKAN(5) models at the low-parameter end and the two ChebyKAN models at the
high-accuracy end. FHKAN(5) with $h=8$ and both ChebyKAN models are more
accurate, smaller, and faster to train than either spline KAN. In agreement
with the single-activation timings of Figure~\ref{fig:time}, the fractional
Hermite block trains about two to three times more slowly than the Chebyshev
basis and two to three times faster than the B-spline edges of comparable
width.

\begin{figure}[htbp]
  \centering
  \includegraphics[width=0.8\textwidth]{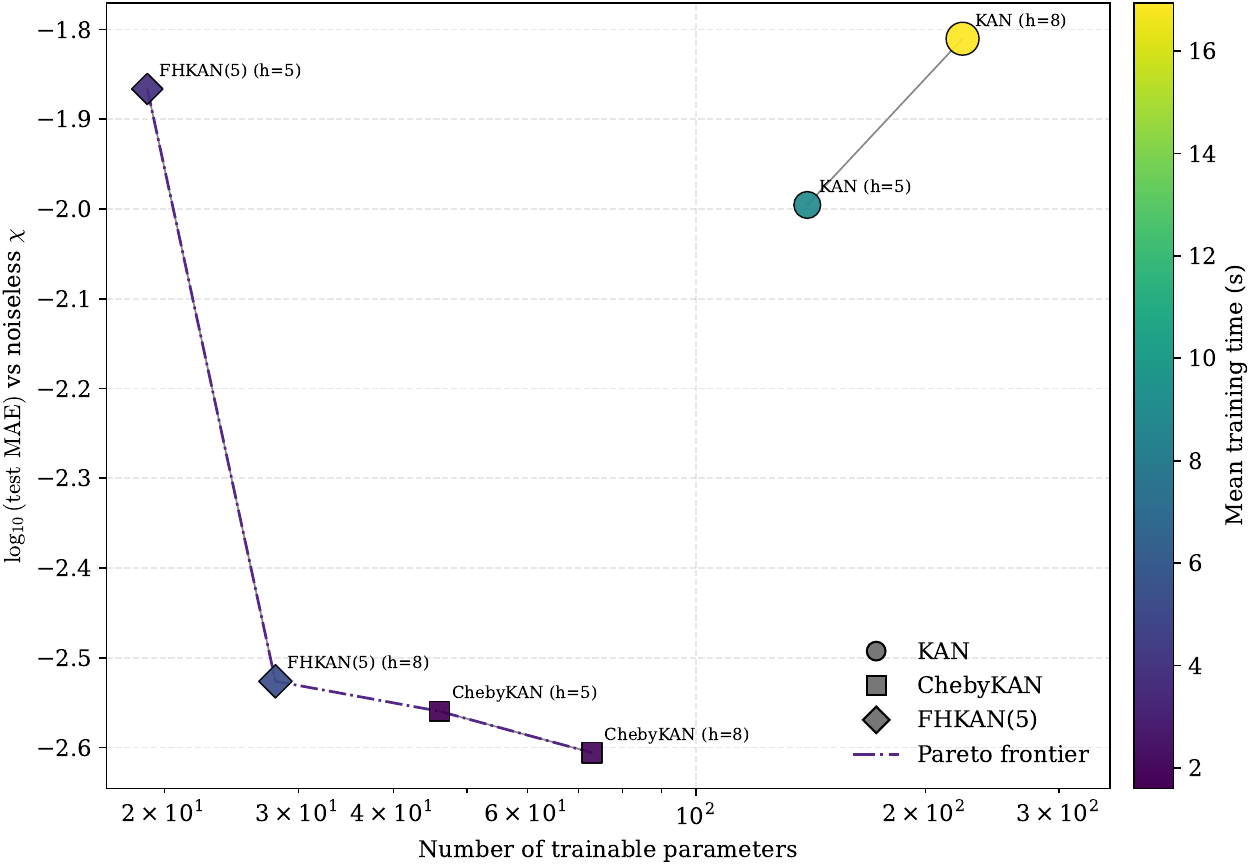}
  \caption{Parameter efficiency on the synthetic regression target, scored
  against the noiseless function. Markers compare a spline KAN, a Chebyshev
  KAN, and FHKAN(5) at hidden widths $h=5$ and $h=8$. The marker color,
  read on the right-hand color bar, gives the mean training time. The Pareto
  frontier collects the models for which no other model is both more accurate and smaller.}
  \label{fig:param-efficiency}
\end{figure}

\FloatBarrier
\section{Conclusion}
\label{sec:conclusion}
We introduced FHKAN, a Kolmogorov--Arnold network whose edge functions are
adaptive, trainable fractional-order generalized Hermite functions. The proposed
fractional Hermite neural block rests on three observations: generalized Hermite
functions are orthonormal and bounded on the entire real line, so they can be
applied directly to unbounded pre-activations without the squashing step required
by Chebyshev-polynomial-based KANs; their Gaussian envelope provides a tunable balance
between local and global behavior, controlled by a trainable scale; and a
trainable fractional exponent continuously enriches the spectral content,
enabling small networks to model non-smooth and singular targets. The derivatives
of the generalized, fractional, and scaled basis are available in closed form
away from the origin, and the order of differentiability of each edge function
at the origin is determined by its shape and fractional parameters, which
identifies the regularity available to physics-informed residuals.

Across a broad evaluation spanning regression, image classification, denoising,
sentiment analysis, and ordinary, partial, and fractional differential equations,
FHKAN consistently improved upon both conventional activation functions and the
vanilla KAN in accuracy and convergence speed, while keeping a compact parameter
budget. Section~\ref{sec:discussion} places that budget on the Pareto frontier
of test error, parameter count, and training time against spline and Chebyshev
KANs. Its principal cost is a modest increase in per-evaluation time relative
to elementary activations---though it remains far cheaper than spline KANs---and,
as with other global spectral bases, reduced
locality of interpretation. Promising directions for future work include
localized or compactly supported fractional Hermite variants, attention-based
fusion of blocks of different orders, and a theoretical analysis of the
approximation rates afforded by trainable fractional orders.

\section*{Declaration of AI-assisted Technologies in The Manuscript Preparation Process}
During the preparation of this manuscript, the authors used Claude to assist with language editing and proofreading purposes, including improving grammar, fluency, and readability of the text. The authors carefully reviewed, verified, and revised all AI-assisted suggestions and take full responsibility for the content of the manuscript.

\section*{Data availability}
The datasets used in this study are publicly available standard benchmarks. The
codes that support the findings are available at https://github.com/ArshaFrn/FHKAN.

\section*{Declaration of competing interest}
The authors declare that they have no known competing financial interests or personal relationships that could have appeared to influence the work reported in this paper.

\bibliographystyle{unsrtnat}
\bibliography{ref}

\end{document}